\documentclass[10pt,journal,compsoc]{IEEEtran}
\ifCLASSOPTIONcompsoc
  \usepackage[nocompress]{cite}
\else
  \usepackage{cite}
\fi
\ifCLASSINFOpdf
\else
\fi
\usepackage{amsmath,amssymb,amsfonts,amsthm,mathtools}
\allowdisplaybreaks[4]
\usepackage{array}
\usepackage{url}

\usepackage[english]{babel}
\usepackage{graphicx}
\usepackage{float}
\usepackage{nomencl}
\usepackage{multirow}
\usepackage{lipsum}
\usepackage{algorithm,algorithmic}
\usepackage{epstopdf}
\usepackage{caption}
\usepackage{subcaption}\usepackage{xcolor}
\usepackage{mathrsfs} 
\usepackage{upgreek}
\usepackage{caption}
\usepackage{graphicx}\usepackage{lettrine}
\usepackage{bm,bbm}
\usepackage{lipsum}
\usepackage{subcaption}
\usepackage{nomencl}\usepackage{stfloats}
\usepackage{url}
\usepackage{booktabs}
\newtheorem{theorem}{Theorem}
\newtheorem*{proof*}{Proof}
\newtheorem{assumption}{Assumption}
\newtheorem{lemma}{Lemma}

\newtheorem{remark}{Remark}
\newtheorem{corollary}{Corollary}

\usepackage{xcolor}
\definecolor{client1}{RGB}{0, 87, 1}
\definecolor{client2}{RGB}{199,53,0}

\usepackage{twemojis}

\title{MimiC: Combating Client Dropouts in Federated Learning by Mimicking Central Updates}

\author{Yuchang~Sun,~\IEEEmembership{Graduate Student Member,~IEEE},
Yuyi~Mao,~\IEEEmembership{Member,~IEEE},
and Jun~Zhang,~\IEEEmembership{Fellow,~IEEE}

\thanks{
      	Y. Sun and J. Zhang are with the Department of Electronic and Computer Engineering, The Hong Kong University of Science and Technology, Hong Kong, China (E-mail: yuchang.sun@connect.ust.hk, eejzhang@ust.hk). 
       
      	Y. Mao is with the Department of Electrical and Electronic Engineering, The Hong Kong Polytechnic University, Hong Kong, China (E-mail: yuyi-eie.mao@polyu.edu.hk).(Corresponding author: Yuyi Mao.)

        Manuscript received 23 June 2023; revised 19 September and 13 November 2023; accepted 20 November 2023.
       
       This work was supported by the Hong Kong Research Grants Council under the Areas of Excellence scheme grant AoE/E-601/22-R and NSFC/RGC Collaborative Research Scheme grant CRS\_HKUST603/22.}
}

\begin{document}

\IEEEtitleabstractindextext{%
\begin{abstract}
Federated learning (FL) is a promising framework for privacy-preserving collaborative learning, where model training tasks are distributed to clients and only the model updates need to be collected at a server. However, when being deployed at mobile edge networks, clients may have unpredictable availability and drop out of the training process, which hinders the convergence of FL. This paper tackles such a critical challenge. Specifically, we first investigate the convergence of the classical FedAvg algorithm with arbitrary client dropouts. We find that with the common choice of a decaying learning rate, FedAvg may oscillate around a stationary point of the global loss function in the worst case, which is caused by the divergence between the aggregated and desired central update. Motivated by this new observation, we then design a novel training algorithm named MimiC, where the server modifies each received model update based on the previous ones. The proposed modification of the received model updates mimics the imaginary central update irrespective of dropout clients. The theoretical analysis of MimiC shows that divergence between the aggregated and central update diminishes with proper learning rates, leading to its convergence. Simulation results further demonstrate that MimiC maintains stable convergence performance and learns better models than the baseline methods. 
\end{abstract}

\begin{IEEEkeywords}
Federated learning (FL), client dropout, straggler effect, edge intelligence, convergence analysis.
\end{IEEEkeywords}

 }
\maketitle

\IEEEraisesectionheading{\section{Introduction}\label{sec:introduction}}

%
%
%
%

\IEEEPARstart{T}{he} resurgence of deep learning (DL) intensifies the demand for training high-quality models from distributed data sources.
However, the traditional approach of centralized model training may raise severe concerns of privacy leakage, since data, possibly with personal and sensitive information, need to be collected by a server prior to training.
As a result, federated learning (FL) \cite{fedavg,bonawitz2019towards}, which is a privacy-preserving distributed model training paradigm, has recently received enormous attention \cite{li2020application,app1,app2,ni1}.
A typical FL system consists of a central server and many clients with private data, which collaborate to accomplish an iterative training process.
In each training iteration, clients train local models based on their local data and upload the model updates to the server. Upon receiving the model updates, the server performs model aggregation and the aggregated global model is disseminated to clients for the next training iteration.

Depending on the types of clients, FL systems can be categorized as the cross-device and cross-silo settings \cite{kairouz2021advances}.
In this work, we focus on cross-device FL \cite{bonawitz2019towards,lim2020federated} where clients are usually edge devices (e.g., smartphones and wearables) with a relatively small volume of local data.
Cross-device FL fits well with the new era of pervasive artificial intelligence (AI) \cite{zhou2019edge,towards}, which utilizes the distributed computational resources for real-time and cost-effective big data analytics.
Note that clients of cross-device FL may have heterogeneous resources such as computing power, battery capacity, and network quality \cite{wang2020optimize,abdelmoniem2022empirical}.
Therefore, they may occasionally drop out of the training process for some iterations because of various unpredictable events, which is termed as \textit{client dropout}.
For example, mobile devices may run out of battery or be disconnected from the server.
In such cases, only a subset of clients are able to complete local training and upload the model updates in each training iteration, which significantly degrades the convergence performance and training speed \cite{9475501}.

\begin{figure}[t]
\centering
\includegraphics[width=\columnwidth]{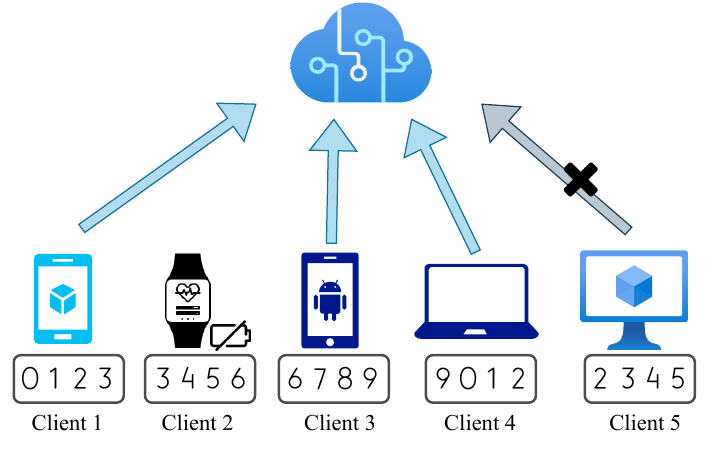} 
\caption{An example of client dropouts in a cross-device FL system with ten classes of data. In this iteration, client $2$ is out of battery while client $5$ is disconnected from the server. The consequence is that information about classes $4$ and $5$ is missing in the received model updates.}
\label{fig:dropout}
\end{figure}

To guarantee satisfactory learning performance, the global dataset is usually partitioned such that training data on clients are independent and identically distributed (IID) in distributed learning. However, while a small fraction of dropout clients can be tolerated with IID training data, significant dropouts severely impair the training efficiency and slow down the convergence speed. 
Techniques such as coded computing \cite{li2020coded,bitar2020stochastic} and using backup clients \cite{xiong2021straggler} combat dropouts in distributed learning by redistributing training data and assigning redundant training tasks to clients.
Nevertheless, client dropouts in FL cause some critical and unique challenges.
Firstly, the local data on clients are non-IID.
The non-IIDness cannot be mitigated since the local data are not allowed to be transferred to other clients or the server because of privacy concerns.
Also, simply ignoring the inactive clients leads to a biased global update in each training iteration.
In other words, if some clients drop out, the aggregated update will be biased towards the training data of the successful clients and thereby deviates from global training objective, which degrades the learned model performance \cite{9963723,9521822}.
Fig. \ref{fig:dropout} shows a simple cross-device FL system with five clients to illustrate this issue, where we assume there are ten classes of data samples in total, but each client has only four classes of data.
We observe that when clients 2 and 5 drop out, the aggregated update fails to exploit the information about classes 4 and 5.

It is worth noting that client dropout is different from active client sampling \cite{yang2021achieving,li2019convergence,fraboni2021impact}, as the latter assumes the server is able to select a portion of clients to participate in each training iteration.
This is plausible due to an implicit premise that all the clients are available to be selected.
However, if some selected clients drop out of training unexpectedly, the server can only aggregate the received updates and obtain a biased update that deviates from the desired global gradient.
One possible approach to deal with client dropouts in cross-device FL is to use the memorized latest updates as a substitute for those dropout clients \cite{yan2020intermittent,mifa}.
Nevertheless, when some clients have been inactive for a long time, the global model may have changed substantially and thus the memorized updates suffer from strong staleness.
In contrast, the deviation of each local update from the central one is primarily determined by the difference between the local and global training objectives. 
Since the objective bias should not change, intuitively the deviation of each update remains relatively stable throughout the training process, as will be explained in Section \ref{sec:mimic}.
As such, we aim to utilize such deviation to correct the individual updates from clients and mimic a central update.

In this paper, we investigate the impacts of arbitrary client dropouts on the convergence performance of FL and propose an effective approach to combat such negative impacts.
Our main contributions are summarized as follows:
\begin{itemize} 
    \item We first investigate the convergence performance of the classical \textit{FedAvg} algorithm \cite{fedavg} with arbitrary client dropouts.
    The theoretical analysis shows that client dropouts introduce a biased update in each training iteration.
    With the common choice of a decaying learning rate, the model learned by FedAvg may oscillate around a stationary point of the global loss function in the worst case.
    \item To combat client dropouts, we propose a novel FL algorithm named \textit{MimiC}.
    The core idea of MimiC is augmenting each received update from the active clients to \underline{mimi}c an imaginary \underline{c}entral update.
    Specifically, each received update is modified by a correction variable derived from previous iterations.
    This can be easily implemented at the server and introduces no additional computation or communication workload to clients.
    \item Despite the modified update in MimiC still being a biased estimate of the global gradient, we prove that its divergence is bounded and diminishes to zero with a proper choice of the learning rates.
    We further characterize the convergence of MimiC by imposing a mild assumption on the maximum number of consecutive client dropout iterations. 
    It is also shown that MimiC converges with a high probability when clients drop out in a probabilistic pattern.
    \item To verify the analysis and test the effectiveness of our proposed algorithm, we simulate a cross-device FL system with a variety of client dropout patterns.
    The extensive experimental results validate the convergence of MimiC with client dropouts in different scenarios.
    It is also observed that MimiC consistently produces better models than the baseline methods.\footnote{The codes are available at \url{https://github.com/hiyuchang/mimic_codes}.}
\end{itemize}

\textbf{Organizations.}
The rest of this paper is organized as follows.
In Section \ref{sec:related}, we introduce the related works. In Section \ref{sec:system}, we describe the system model and setup of the cross-device FL problem.
We analyze the convergence behavior of FedAvg and show the negative impacts of arbitrary client dropouts in Section \ref{sec:fedavg}.
To combat client dropouts, we propose a training algorithm named MimiC in Section \ref{sec:mimic} and demonstrate its convergence in Section \ref{sec:convergenceofmimic}.
The proposed algorithm is evaluated via extensive simulations in Section \ref{sec:evaluation}. Finally, we conclude the paper in Section \ref{sec:conclusion}.

\section{Related Works}\label{sec:related}

In this section, we review the exiting works on two research directions of FL that are closely related to our current study, including: 1) cross-device FL with partial client participation; and 2) techniques combating client dropouts in distributed and federated learning.

\textbf{Cross-device FL with partial client participation.}
The cross-device FL systems usually consist of a large number of clients and only a fraction of them are able to participate in the training process at each iteration due to the limited communication and computational resources \cite{kairouz2021advances,9475501,9237168}.
To ensure model convergence, there are two popular client sampling strategies, including: 1) to sample clients with replacement with respect to preset probabilities; 2) to sample clients without replacement uniformly at random.
These strategies guarantee that the weighted sum of the sampled clients' objectives is an unbiased estimate of the global training objective.
However, in realistic applications, it is hard to implement these strategies due to arbitrary client dropouts, which makes the training objective biased in the presence of non-IID data.

Recently, Ruan \textit{et al.} \cite{ruan2021towards} investigated the convergence performance of FedAvg with strongly convex loss functions in practical scenarios where clients have various participating behaviors.
Specifically, the clients are allowed to have various local steps in each training iteration and those clients with a local step of zero can be viewed as dropout clients.
It is shown that by controlling the aggregation weights FedAvg with client dropouts can still converge to the global optimum, which, however, required the prior knowledge of client participation probabilities.
Besides, a unified convergence analysis for FL with arbitrary client participation was provided in \cite{wang2022unified}, where the accumulated model updates in every $H$ training iterations are adopted as momentum to accelerate the global model training. It was shown that the arbitrary client participation introduces an additional term $\tilde{\delta}^{2}(H)$ in the convergence bound, which can be eliminated only if the client participation pattern, the amplification period $H$, and the aggregations weights satisfy certain conditions, e.g., regularized participation or ergodic participation.
In practice, these conditions are difficult to verify and need laborious hyperparameter tuning.
In addition, another thread of works \cite{9237168,shi2022vfedcs} designed various client sampling strategies to accelerate the convergence rate of FL.
These works, however, are beyond the scope of this paper, since they assume clients can be actively selected to participate in the training process. 

\textbf{Combating client dropouts in distributed and federated learning.}
In distributed learning, client dropouts can be handled via coded computing \cite{li2020coded,karakus2019redundancy}, which carefully designs the training tasks on clients to introduce some levels of redundancy, such that the dropout clients can no longer cause fatal effects on training.
However, these methods either assume IID data or require to redistribute the training data among clients, which may be contradictory to the privacy-preserving principle of FL.
Motivated by coded computing for distributed learning, recent works \cite{CFL_journal,scfl} proposed to construct coded datasets at clients, which are then transmitted to the server or other clients to compensate for the missing model updates of dropout clients.
Nevertheless, coded data transmission may bring additional communication overhead and be risky of privacy leakage. 
Similarly, \cite{schlegel2023codedpaddedfl,shao2022dres} resorted to Shamir’s secret sharing for the coded data, which retains the privacy protection of conventional FL.
However, the substantial communication overhead in these designs is intolerable especially with thousands of clients in cross-device FL systems \cite{kairouz2021advances}.

In addition to coded computing techniques, some other strategies can be adopted to combat client dropouts in FL.
The simplest strategy of combating client dropouts in FL is to only aggregate the received updates. This passive approach, however, significantly deteriorates the convergence performance due to the biased aggregated update.
Besides, another strategy is to adopt the latest updates from inactive clients as substitutes \cite{yan2020intermittent,mifa}.
Nevertheless, the memorized updates with significant staleness may degrade the training performance in the current iteration.
Alternatively, the local update from another client with similar data can be used as substitute for that of a dropout client \cite{wang2022friends}. However, this approach requires to compute the pair-wise similarity between any two clients, which leads to a high computation cost.
We also notice that there is another line of studies \cite{jiang2022taming,liu2022efficient} on privacy-preserving schemes for cross-device FL with client dropouts.
Specifically, to achieve a target level of differential privacy (DP) in FL with random client dropouts, Jiang \textit{et al.} \cite{jiang2022taming} designed an “add-then-remove” approach to determine the noise levels of local gradients.
Moreover, a privacy-preserving aggregation scheme for FL was proposed in \cite{liu2022efficient}, which is resilient to client dropouts with the aid of Shamir’s secret sharing.
However, these works have not implemented any mechanism to compensate the missing model updates of dropout clients and still suffer from degraded model performance.

\section{System Model and Problem Setup}\label{sec:system}
We consider a cross-device FL system consisting of a central server and $N$ clients (denoted by the set $\mathcal{N}=\{1,\dots,N\}$).
These nodes collaboratively train a model $\mathbf{w}\in \mathbb{R}^{M}$ with $M$ trainable parameters to minimize the loss over data samples of all clients.
The global training objective is expressed as follows:
\begin{equation}
    \min_{\mathbf{w}\in \mathbb{R}^{M}} f(\mathbf{w}) = \frac{1}{N} \sum_{i=1}^N f_i(\mathbf{w}),
\end{equation}
where $f_i(\mathbf{w})$ denotes the $i$-th client's loss function over the local dataset $\mathcal{D}_i$.

To optimize model $\mathbf{w}$, the classical FedAvg \cite{fedavg} algorithm divides the training process of model $\mathbf{w}$ into $T$ iterations.
In iteration $t\in[T]$, clients receive a global model $\mathbf{w}_{t}$ from the server and train the model locally for $K$ steps using the mini-batch SGD algorithm.
Specifically, the local model is initialized as $\mathbf{w}_{t,0}^i \equiv \mathbf{w}_{t}, \forall i\in\mathcal{N}$.
Since a client may own a considerable amount of training data, it is impractical to compute the gradient over the whole dataset for local training.
Thus, client $i$ samples a batch of data $\xi$ and computes gradient $\nabla F_i(\mathbf{w}; \mathbf{\xi})$ based on the current model $\mathbf{w}$.
At the $k$-th local step, the client uniformly samples a batch of training data $\mathbf{\xi}_{t,k}^i$, and updates the model $\mathbf{w}_{t,k}^i$ as follows:
\begin{equation}
    \mathbf{w}_{t,k+1}^i \leftarrow \mathbf{w}_{t,k}^i - \eta_L \nabla F_i(\mathbf{w}_{t,k}^i; \mathbf{\xi}_{t,k}^i), k = 0,1,\dots,K \!- \! 1,
    \label{eq:local-sgd}
\end{equation}
where $\eta_L$ denotes the local learning rate.
The local update after $K$ local steps is summarized as follows:
\begin{equation}
\hat{g}_i(\mathbf{w}_t)\triangleq\frac{1}{K}\sum_{k=0}^{K-1} \nabla F_i(\mathbf{w}_{t,k}^i; \mathbf{\xi}_{t,k}^i),
\end{equation}
which is then uploaded to the server.

In each iteration, we assume that only a subset of clients $\mathcal{S}_{t}$ is active for training, which are denoted as \textit{active clients}.
Others are dropout clients (i.e., inactive clients) that fail in either local training or gradient uploading.
The server only receives model updates from the active clients in $\mathcal{S}_t$.
Afterwards, the global update is computed as an average of them as follows:
\begin{equation}
    \mathbf{v}_t \triangleq \frac{1}{|\mathcal{S}_t|} \sum_{i\in\mathcal{S}_t} \hat{g}_i(\mathbf{w}_t).
\end{equation}
The global model is then updated using the global learning rate $\eta_t$ as follows:
\begin{equation}
    \mathbf{w}_{t+1} \leftarrow \mathbf{w}_t - \eta_t \mathbf{v}_t.
    \label{eq:global-update}
\end{equation}
This new global model $\mathbf{w}_{t+1}$ is broadcast to all clients for the next training iteration.

In the training process, each client may remain inactive for several consecutive training iterations.
We introduce the notation of $\tau(t,i)$, which indicates client $i$ has been inactive for $\tau(t,i)-1$ consecutive iterations, i.e., the last iteration of client $i$ being active before iteration $t$ is $t-\tau(t,i)$ \cite{mifa,yan2020intermittent}.
Fig. \ref{fig:staleness} provides an illustration of the client availability in cross-device FL systems.
\begin{assumption}\label{ass:clients}
(Client availability)
There exists a constant $\tau_{\text{max}} > 0$ such that $\tau(t,i) \leq \tau_{\text{max}},\forall i\in\mathcal{N},t\in[T]$.
\end{assumption}

\begin{figure}[t]
    \centering
    \includegraphics[width=\columnwidth]{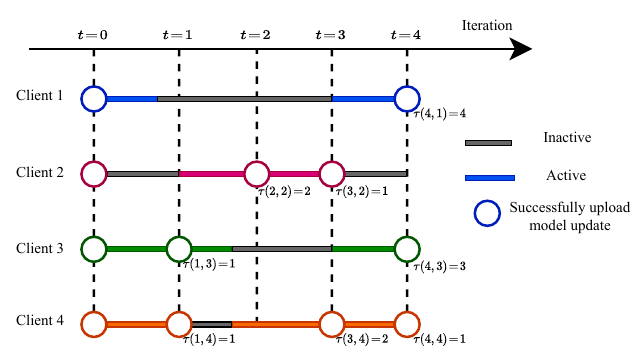}
    \caption{Illustration of the client availability in cross-device FL, where $\tau_{\text{max}}=3$.}
    \label{fig:staleness}
\end{figure}

Besides, to facilitate the theoretical analysis and algorithm development, we make the following common assumptions on the loss functions \cite{wang2020tackling,li2019convergence,yang2021achieving,fedprox}.
\begin{assumption}\label{ass-1}
($L$-smoothness)
There exists a constant $L>0$ such that for any $\mathbf{w}_1,\mathbf{w}_2\in\mathbb{R}^{M}$, $\| \nabla f_i(\mathbf{w}_1) - \nabla f_i(\mathbf{w}_2) \|_2 \leq L \| \mathbf{w}_1 - \mathbf{w}_2 \|_2, \forall i\in\mathcal{N}$.
\end{assumption}

\begin{assumption}\label{ass-2}
(Unbiased and variance-bounded stochastic gradient)
The stochastic gradient $ \nabla F_i(\mathbf{w}; \mathbf{\xi})$ on a randomly sampled batch of data $\mathbf{\xi} \in \mathcal{D}_i$ is an unbiased estimate of the full-batch gradient, i.e., $\mathbb{E}\left[ \nabla F_i(\mathbf{w}; \mathbf{\xi}) \right] = \nabla f_i(\mathbf{w}), \forall i\in\mathcal{N}$.
Besides, there exists a constant $\sigma>0$ such that $\mathbb{E}\left[ \| \nabla F_i(\mathbf{w}; \mathbf{\xi}) - \nabla f_i(\mathbf{w}) \|_2^2 \right] \leq \sigma^2, \forall i\in\mathcal{N}$.
\end{assumption}

\begin{assumption}\label{ass-3}
(Data heterogeneity) 
There exist constants $\kappa_i >0, \forall i \!\in\! \mathcal{N}$ such that $\|\nabla f_{i}(\mathbf{w}) -\nabla f(\mathbf{w})\|_2^{2} \leq \kappa_i^2$.
\end{assumption}

In the next section, we analyze the convergence performance of FedAvg and show the negative effect of arbitrary client dropouts on its convergence.

\section{Theoretical Analysis of FedAvg with Arbitrary Client Dropouts}\label{sec:fedavg}

The convergence of FedAvg, under the assumption of full client participation or uniform partial participation, has been presented in \cite{scaffold,9796818,li2019convergence}. Specifically, if the global learning rates are selected as $\eta_t \equiv \mathcal{O}\left( \frac{1}{\sqrt{T}}\right)$, FedAvg is guaranteed to converge to a stationary point of the global loss function with a rate of $\mathcal{O}\left(\frac{1}{\sqrt{T}}\right)$.
However, these convergence results rely on strict assumptions of client availability, which ensures an unbiased estimate of the global gradient at the server, i.e.,
\begin{equation}
    \mathbb{E}[\mathbf{v}_t] = \mathbb{E}_{\mathcal{S}_t}\left[ \frac{1}{|\mathcal{S}_t|} \sum_{i\in\mathcal{S}_t} \hat{g}_i(\mathbf{w}_t) \right] = \frac{1}{|\mathcal{N}|} \sum_{i\in\mathcal{N}} \hat{g}_i(\mathbf{w}_t).
    \label{eq:unbiased}
\end{equation}
Since client dropouts are spontaneous and uncontrollable, \eqref{eq:unbiased} can hardly be achieved in practice.
In this section, we investigate the convergence behavior of FedAvg in the presence of arbitrary client dropouts, which relaxes the assumption of full or uniform partial participation.

To proceed, we first derive an upper bound of the loss decay in each training iteration in Lemma \ref{lem:one-round}, which is applicable to any FL algorithm satisfying the global update scheme in \eqref{eq:global-update}.
\begin{lemma}\label{lem:one-round}
For any FL algorithm with the global model update scheme in \eqref{eq:global-update}, the loss decay in each iteration is upper bounded as follows:
\begin{align}
    & \mathbb{E} [f(\mathbf{w}_{t+1})] - \mathbb{E} [f(\mathbf{w}_{t})]     \label{eq:one-round} \\
    \leq & - \frac{\eta_t}{2} \mathbb{E} \left[\left\| \nabla f(\mathbf{w}_t) \right\|^2 \right] - \left(\frac{\eta_t}{2} - \frac{\eta_t^2 L}{2}\right) \left\| \mathbb{E}[\mathbf{v}_t] \right\|^2  \nonumber \\
    + & \frac{\eta_t^2 L}{2} \underbrace{\mathbb{E} \left[\| \mathbf{v}_t - \mathbb{E} [\mathbf{v}_t]\|^2 \right]}_{\Phi_t}
    + \frac{\eta_t}{2} \underbrace{\mathbb{E} \left[\left\| \mathbb{E}[\mathbf{v}_t] - \nabla f (\mathbf{w}_t) \right\|^2 \right]}_{\mathcal{E}_t}. \nonumber
\end{align}
\end{lemma}
\begin{proof*}
Recall in \eqref{eq:global-update} that $\mathbf{w}_{t+1} = \mathbf{w}_t - \eta_t \mathbf{v}_t$ and accordingly $f(\mathbf{w}_{t+1}) = f(\mathbf{w}_t - \eta_t \mathbf{v}_t)$.
By leveraging the $L$-smoothness of function $f(\cdot)$ (i.e., Assumption \ref{ass-1}), we have:
\begin{align}
    & \mathbb{E} [f(\mathbf{w}_{t+1})] - \mathbb{E} [f(\mathbf{w}_{t})] \nonumber \\
    \leq & -\eta_t \mathbb{E} \left\langle \nabla f(\mathbf{w}_{t}), \mathbf{v}_t \right\rangle + \frac{\eta_t^2 L}{2} \mathbb{E} \left[\| \mathbf{v}_t \|^2 \right] \nonumber \\
    = & -\eta_t \mathbb{E} \left\langle \nabla f(\mathbf{w}_{t}), \mathbb{E}[\mathbf{v}_t] \right\rangle + \frac{\eta_t^2 L}{2} \mathbb{E} \left[\| \mathbf{v}_t \|^2 \right] \nonumber \\
    \overset{(\text{a})}{=} & - \frac{\eta_t}{2} \mathbb{E} \left[\left\| \nabla f(\mathbf{w}_t) \right\|^2 \right] - \frac{\eta_t}{2} \left\| \mathbb{E}[\mathbf{v}_t] \right\|^2 \nonumber \\
    & + \frac{\eta_t}{2} \mathbb{E} \left[\left\| \mathbb{E}[\mathbf{v}_t] - \nabla f (\mathbf{w}_t) \right\|^2 \right] + \frac{\eta_t^2 L}{2} \mathbb{E} \left[\| \mathbf{v}_t \|^2 \right] \nonumber \\
    \overset{(\text{b})}{=} & - \frac{\eta_t}{2} \mathbb{E} \left[\left\| \nabla f(\mathbf{w}_t) \right\|^2 \right] + \frac{\eta_t}{2} \underbrace{\mathbb{E} \left[\left\| \mathbb{E}[\mathbf{v}_t] - \nabla f (\mathbf{w}_t) \right\|^2 \right]}_{\mathcal{E}_t} \nonumber \\
    & - \left(\frac{\eta_t}{2} - \frac{\eta_t^2 L}{2} \right) \left\| \mathbb{E}[\mathbf{v}_t] \right\|^2 + \frac{\eta_t^2 L}{2} \underbrace{\mathbb{E} \left[\| \mathbf{v}_t - \mathbb{E} [\mathbf{v}_t]\|^2 \right]}_{\Phi_t},
\end{align}
where (a) is due to the fact that $\left\langle \mathbf{x},\mathbf{y} \right\rangle = \frac{1}{2} \| \mathbf{x} \|^2 + \frac{1}{2} \| \mathbf{y} \|^2 - \frac{1}{2} \| \mathbf{x}-\mathbf{y} \|^2, \forall \mathbf{x}, \mathbf{y}$, and (b) is because for any random variable $\mathbf{v}$ with expectation $\mathbb{E}[\mathbf{v}]$, $\mathbb{E}[\| \mathbf{v} \|^2] = \| \mathbb{E}[\mathbf{v}] \|^2 + \| \mathbb{E}[\mathbf{v}] - \mathbf{v} \|^2$ holds.
\end{proof*}

\begin{remark}\label{rm-1}
Lemma \ref{lem:one-round} shows that the loss decay in each iteration is bounded by the weighted sum of $\mathcal{E}_{t}$ and $\Phi_{t}$. For clarity, we use superscript $\mathrm{F}$ and $\mathrm{M}$ to represent these quantities in FedAvg and MimiC, respectively. As the values of $\Phi_t$ and $\mathcal{E}_t$ increase, the loss decay of each training iteration reduces, implying their negative effects on the convergence performance.
Note that the term $\Phi_t$ denotes the error caused by mini-batch data sampling, which is common in both FL and distributed training with SGD as the local training algorithm \cite{stich2018local,haddadpour2019local}.
Moreover, the term $\mathcal{E}_t$ represents the \textit{gradient estimation error} between the applied global update $\mathbf{v}_t$ and the actual global gradient over all data samples $\nabla f(\mathbf{w}_t)$.
\end{remark}

We are now to provide an upper bound for the gradient estimation error $\mathcal{E}_t$ in FedAvg.
As will be shown later in Section \ref{sec:mimic}, the value of $\mathcal{E}_t$ serves as a key difference between FedAvg and our proposed algorithm.

\begin{lemma}\label{lem:E_t_1}
If the local learning rate satisfies $\eta_L \leq \frac{1}{10 L}$, the gradient estimation error $\mathcal{E}_t^{\mathrm{F}}$ in FedAvg is upper bounded as follows:
\begin{equation}
\begin{split}
    \mathcal{E}_t^{\mathrm{F}}
    &\leq \frac{2L^2}{|\mathcal{S}_t|} \sum_{i\in\mathcal{S}_t} \left[ 8\eta_L^2 K(K-1) \mathbb{E}[\| \nabla f(\mathbf{w}_{t}) \|^2] \right. \\
    &+ \left. 8\eta_L^2 K(K-1) \kappa_i^2 + 2\eta_L^2 (K-1) \sigma^2 \right] \\
    &+ 2 \underbrace{\mathbb{E} \left[ \left\| \frac{1}{|\mathcal{S}_t|} \sum_{i\in\mathcal{S}_t} \nabla f_i(\mathbf{w}_{t}) - \nabla f(\mathbf{w}_{t}) \right\|^2 \right]}_{\gamma_t},
\end{split}
\label{eq:E_t-fedavg}
\end{equation}
where $\gamma_t$ denotes the objective bias in the $t$-th training iteration, i.e., the gap between the aggregated gradient and the global gradient $\nabla f(\mathbf{w}_{t})$.
\end{lemma}
\begin{proof}[Proof Sketch]
Recall the definition of $\mathcal{E}_t^{\mathrm{F}}$ in FedAvg as:
\begin{equation}
    \mathcal{E}_t^{\mathrm{F}} = \mathbb{E} \left[ \left\| \frac{1}{|\mathcal{S}_t|} \sum_{i\in\mathcal{S}_t} \frac{1}{K} \sum_{k=0}^{K-1} \nabla f_i(\mathbf{w}_{t,k}^i) - \nabla f(\mathbf{w}_{t}) \right\|^2 \right].
    \label{eq:help-10}
\end{equation}
We first expand \eqref{eq:help-10} by adding and subtracting an auxiliary term $\frac{1}{|\mathcal{S}_t|} \sum_{i\in\mathcal{S}_t} \frac{1}{K} \sum_{k=0}^{K-1} \nabla f_i(\mathbf{w}_{t})$ and further upper bound it using the inequality $\| \mathbf{x}_1 + \mathbf{x}_2\|_2^2 \leq 2 \| \mathbf{x}_1\|_2^2 + 2\| \mathbf{x}_2\|_2^2$ as:
\begin{align}
    \mathcal{E}_t^{\mathrm{F}} \leq & 2 \mathbb{E} \left[ \left\| \frac{1}{|\mathcal{S}_t|} \sum_{i\in\mathcal{S}_t} \frac{1}{K} \sum_{k=0}^{K-1} \left( \nabla f_i(\mathbf{w}_{t,k}^i) - \nabla f_i(\mathbf{w}_{t}) \right)  \right\|^2 \right] \nonumber \\
    + & 2 \mathbb{E} \left[ \left\| \frac{1}{|\mathcal{S}_t|} \sum_{i\in\mathcal{S}_t} \nabla f_i(\mathbf{w}_{t}) - \nabla f(\mathbf{w}_{t}) \right\|^2 \right].
    \label{eq:help-11}
\end{align}
The first term on the right-hand side (RHS) of \eqref{eq:help-11} quantifies  the difference between local gradients on $\mathbf{w}_{t,k}^i$ and $\mathbf{w}_{t}$, which can be upper bounded using a similar proof of Lemma C.2 in \cite{mifa}.
Please refer to Appendix \ref{proof:lem:E_t_1} for the detailed proof.
\end{proof}

Next, we derive an upper bound for $\Phi_t^{\mathrm{F}}$ in FedAvg in the following lemma.
\begin{lemma}\label{lem:phi_t_1}
In FedAvg, we have:
\begin{equation}
    \Phi_t^{\mathrm{F}} \leq \frac{\sigma^2}{|\mathcal{S}_t|K}. 
    \label{eq:fedavg:phi}
\end{equation}
\end{lemma}
\begin{proof}
Following the proof of Lemma 2 in \cite{wang2020tackling}, we decompose the error as follows:
\begin{small}
\begin{align}
    \Phi_t^{\mathrm{F}}
    = & \frac{1}{|\mathcal{S}_t|^2} \sum_{i\in\mathcal{S}_t} \mathbb{E} \left[ \left\| \hat{g}_i(\mathbf{w}_t) - \mathbb{E} \left[ \hat{g}_i(\mathbf{w}_t) \right] \right\|^2 \right] \nonumber \\
    = & \frac{1}{|\mathcal{S}_t|^2K^2} \! \sum_{i\in\mathcal{S}_t} \! \sum_{k=0}^{K-1} \! \mathbb{E} \! \left[\! \left\| \nabla F_i(\mathbf{w}_{t,k}^i; \mathbf{\xi}_{t,k}^i) \!-\! \mathbb{E} \! \left[ \nabla F_i(\mathbf{w}_{t,k}^i; \mathbf{\xi}_{t,k}^i) \right] \right\|^2 \! \right] \nonumber \\
    \overset{(\text{a})}{\leq} & \frac{\sigma^2}{|\mathcal{S}_t|K},
\end{align}
\end{small}
where (a) follows Assumption \ref{ass-2}.
\end{proof}

After respectively bounding the terms $\Phi_t^{\mathrm{F}}$ and $\mathcal{E}_t^{\mathrm{F}}$, we obtain the following theorem that characterizes the convergence behavior of FedAvg with arbitrary client dropouts.
\begin{theorem}\label{convergence-fedavg}
With Assumptions \ref{ass-1}-\ref{ass-3}, if the local learning rate satisfies $1-C_K>0$ and $\eta_L = \mathcal{O}\left( \frac{1}{L\sqrt{TK}} \right) \leq \frac{1}{10 L}$, the following inequality holds:
\begin{align}
    & \frac{1}{\Xi_T} \sum_{t=0}^{T-1} \eta_t (1 - C_K ) \mathbb{E} [\left\| \nabla f(\mathbf{w}_t) \right\|^2] \nonumber \\
    \leq & \frac{ 2 \Delta }{\Xi_T}
    + \frac{1}{\Xi_T} \sum_{t=0}^{T-1} \left( \frac{\eta_t^2 L}{|\mathcal{S}_t| K} + 4\eta_t \eta_L^2 L^2 (K-1) \right) \sigma^2 \nonumber \\
    + & \frac{1}{\Xi_T} \sum_{t=0}^{T-1} 16\eta_t \eta_L^2 L^2 K(K-1) \frac{1}{|\mathcal{S}_t|} \sum_{i\in\mathcal{S}_t} \kappa_i^2 \nonumber \\
    + & \underbrace{\frac{1}{\Xi_T} \sum_{t=0}^{T-1} \eta_t \mathbb{E}\left[ \left\| \frac{1}{|\mathcal{S}_t|} \sum_{i\in\mathcal{S}_t} \nabla f_i(\mathbf{w}_t) - \nabla f(\mathbf{w}_t) \right\|^2 \right] }_{\Gamma_T},
    \label{eq:fedavg}
\end{align}
where $C_K \triangleq 16 \eta_L^2 L^2 K(K-1)$, $\Xi_T \triangleq \sum_{t=0}^{T-1} \eta_t$, and $\Delta \triangleq \mathbb{E} [f(\mathbf{w}_{0})] - \mathbb{E} [f(\mathbf{w}^*)]$.
\end{theorem}
\begin{proof}[Proof Sketch]
We rearrange the terms in \eqref{eq:one-round}, substitute the terms $\mathcal{E}_t^{\mathrm{F}}$ and $\Phi_t^{\mathrm{F}}$ by their upper bounds in \eqref{eq:E_t-fedavg} and \eqref{eq:fedavg:phi}, respectively, and sum up both sides over $t=0,1,\cdots,T-1$ to obtain the result.
Please refer to Appendix \ref{proof:thm:fedavg} for the detailed proof.
\end{proof}

\begin{remark}
At the RHS of \eqref{eq:fedavg}, the first term measures the initial error.
The second and third terms are accumulated errors respectively caused by the mini-batching sampling and data heterogeneity. As aforementioned, these terms are common in the convergence bounds and cannot be mitigated in local SGD-based FL \cite{khaled2020tighter,stich2018local}.
Moreover, the term $\Gamma_T$ represents the \textit{objective mismatch} caused by client dropouts over the $T$ iterations.
\end{remark}

\begin{remark}
\textbf{(Convergence of FedAvg)}
The common choice of the global learning rate in existing FL literature satisfies $\lim_{T\rightarrow \infty} \sum_{t=0}^{T-1} \eta_t = \infty$, $\lim_{T\rightarrow \infty} \sum_{t=0}^{T-1} \eta_t^2 < \infty$ \cite{fedprox,li2019convergence,bottou2018optimization}.
Notice that with such global learning rates, the RHS of \eqref{eq:fedavg} except the term $\Gamma_T$ converges to zero as $T\rightarrow \infty$, implying that in the presence of arbitrary client dropout, FedAvg may only oscillate around a stationary point of the global loss function in the worst case. Nevertheless, there are still two obvious special cases where FedAvg can converge with client dropouts: 1) data among clients are IID (i.e., $\nabla f(\mathbf{w}) \equiv \nabla f_i(\mathbf{w}), \forall i\in\mathcal{N}$); and 2) clients drop out with the same probability ($\Gamma_T=0$) \cite{li2019convergence,yang2021achieving}.
Unfortunately, none of these cases is practical in cross-device FL.
\end{remark}

Since the objective mismatch $\Gamma_T$ caused by arbitrary client dropouts cannot be eliminated by FedAvg, in order to mitigate its negative effects on model training, we propose a novel FL algorithm in the next section by modifying the update for each active client to mimic the imaginary central update.

\section{Proposed Algorithm: MimiC}\label{sec:mimic}
In this section, we propose the MimiC algorithm to mitigate the objective mismatch problem caused by client dropouts.
Since each client has a distinctive local objective, we introduce a \textit{correction variable} (denoted by $\mathbf{c}_t^i$) in MimiC to correct its local update such that the aggregated update is close to a desired \textit{central update} $\nabla f(\mathbf{w}_t)$.
Note that $\nabla f(\mathbf{w}_t) = \frac{1}{N}\sum_{i\in\mathcal{N}} \nabla f_i(\mathbf{w}_t)$ represents an imaginary update over all training data based on the current model $\mathbf{w}_t$, which cannot be obtained with client dropouts.
Specifically, once any client $i$ uploads the local update $\hat{g}_i(\mathbf{w}_t)$ successfully, the server modifies the received update as $\mathbf{v}_t^i \triangleq \hat{g}_i(\mathbf{w}_t) + \mathbf{c}_t^i$. The global update is computed as $\mathbf{v}_t = \frac{1}{|\mathcal{S}_t|} \sum_{i\in\mathcal{S}_t} \mathbf{v}_t^i$ and applied to the global model according to \eqref{eq:global-update}.
To ensure model convergence, the key is to design this correction variable properly.

\begin{algorithm}[tb]
\caption{Training Algorithm of MimiC}\label{alg:algorithm}
\textbf{Operations at the server side:}
\begin{algorithmic}[1] 
\STATE{Initialize $\mathbf{c}_0^i=\mathbf{0}, \forall i\in\mathcal{N}$ and a random model $\mathbf{w}_0$;}
\FOR{$t=0,1,\dots,T-1$}
    \STATE{Broadcast $\mathbf{w}_{t}$ to all clients and inform them to perform \textit{LocalTrain}($\mathbf{w}_t$);}
	\STATE{Receive updates $\hat{g}_i(\mathbf{w}_t)$ from active clients in $\mathcal{S}_t$;}
	\FOR{each received update $\hat{g}_i(\mathbf{w}_t), i\in\mathcal{S}_t$}
		\STATE{Modify the update as $\mathbf{v}_t^i \triangleq \hat{g}_i(\mathbf{w}_t) + \mathbf{c}_{t-\tau(t,i)}^i$;}
	\ENDFOR
	\STATE{Compute the global update as $\mathbf{v}_t \triangleq \frac{1}{|\mathcal{S}_t|} \sum_{i\in\mathcal{S}_t} \mathbf{v}_t^i$;}
	\STATE{Update the global model as $\mathbf{w}_{t+1} = \mathbf{w}_t - \eta_t \mathbf{v}_t$;}
	\STATE{Update the correction variable $\mathbf{c}_t^i \leftarrow \mathbf{v}_t - \hat{g}_i(\mathbf{w}_t)$ for each $i\in\mathcal{S}_t$;}
\ENDFOR
\end{algorithmic}
\textbf{Operations at the client side:\\}
\textbf{def} \textit{LocalTrain}($\mathbf{w}_t$):
\begin{algorithmic}[1]
\STATE{Initialize $\mathbf{w}_{t,0}^i=\mathbf{w}_{t}$;}
\FOR{$k=0,1,\dots,K-1$}
	\STATE{Perform local training according to $\mathbf{w}_{t,k+1}^i \leftarrow \mathbf{w}_{t,k}^i - \eta_L \nabla F_i(\mathbf{w}_{t,k}^i; \mathbf{\xi}_{t,k}^i)$;}
\ENDFOR
\STATE{Compute the local update as $\hat{g}_i(\mathbf{w}_t) \triangleq \frac{1}{K}\sum_{k=0}^{K-1} \nabla F_i(\mathbf{w}_{t,k}^i; \mathbf{\xi}_{t,k}^i)$;}
\STATE{\textbf{Return} $\hat{g}_i(\mathbf{w}_t)$ to the server.}
\end{algorithmic}
\end{algorithm}

Given the limited resources of clients in cross-device FL, we should avoid introducing additional computation or communication workload to them.
Besides, since the server cannot perceive the client dropouts in advance, it is difficult to take preventative actions, e.g., assigning a higher training workload to a client that may become unavailable later.
Therefore, we propose to exploit the gathered updates at the server by monitoring the training drifts in previous iterations.
Assuming client $i$ is active in training iteration $t$, the correction variable is defined as the drift from the current local update $\hat{g}_i(\mathbf{w}_t)$ to the applied global update $\mathbf{v}_t$, i.e.,
\begin{equation}
    \mathbf{c}_t^i \triangleq \mathbf{v}_t - \hat{g}_i(\mathbf{w}_t).
    \label{eq:c}
\end{equation}
The correction variable $\mathbf{c}_t^i$ is updated if client $i$ is active in an iteration; otherwise it remains unchanged.
In a later iteration, the server utilizes $\{\mathbf{c}_t^i\}$'s to correct the biased client updates.
To summarize, the global update in iteration $t$ of MimiC is expressed as follows:
\begin{align}
    \mathbf{v}_t &\triangleq \frac{1}{|\mathcal{S}_t|} \sum_{i\in\mathcal{S}_t} \left( \hat{g}_i(\mathbf{w}_t) + \mathbf{c}_{t-\tau(t,i)}^{i} \right) \nonumber \\
    &= \frac{1}{|\mathcal{S}_t|} \sum_{i\in\mathcal{S}_t} \left( \hat{g}_i(\mathbf{w}_t) - \hat{g}_i\left(\mathbf{w}_{t-\tau(t,i)}\right) + \mathbf{v}_{t-\tau(t,i)} \right).
\end{align}

The detailed training process of MimiC is summarized in Algorithm \ref{alg:algorithm}.
We also show an example with four clients to illustrate the implementation of MimiC in Fig. \ref{fig:correction}, where two of them are available in the $t$-th training iteration.
Each active client uses its correction variable maintained from the latest active iteration $t^\prime = t-\tau(t, i)$ to correct its local update.
As a result, the aggregated update, i.e., the average of their corrected updates, is forced to mimic an imaginary central update and thus the global model is optimized towards the minimum of the global training objective. 
\begin{figure*}[t]
\centering
\includegraphics[width=0.7\textwidth]{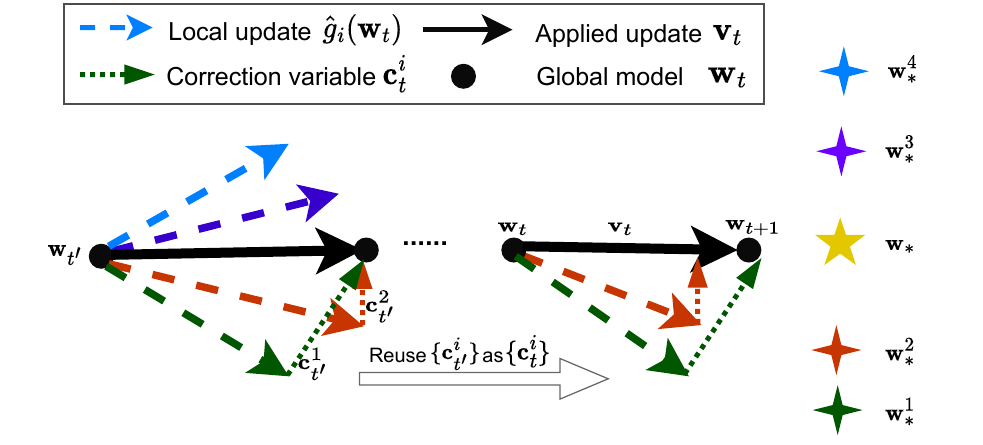} 
\caption{An illustration of MimiC with fours clients. 
While $\mathbf{w}_*$ is the optimum of the global loss function, clients perform updates towards the optimums of their local loss functions (denoted by $\{\mathbf{w}^i_*\}$'s).
For clarity, consider both client $1$ and client $2$ are active in iteration $t^{\prime}$ and $t$.
In iteration $t$, $\mathbf{c}_{t^{\prime}}^{1}$ and $\mathbf{c}_{t^{\prime}}^{2}$ are used to correct their updates. The average of the modified updates gives the applied global update $\mathbf{v}_t$.}
\label{fig:correction}
\end{figure*}

To demonstrate the rationale of the proposed MimiC algorithm, we consider a simplified case with full-batch gradient descent as the local optimizer and assume one local step ($K=1$) in each training iteration, i.e., $\hat{g}_i(\mathbf{w}_t) \equiv \nabla f_i(\mathbf{w}_t)$.
Denote $\Delta_{-i}(\mathbf{w}_{t}) \triangleq \nabla f(\mathbf{w}_t) - \nabla f_i(\mathbf{w}_t)$.
The modified update $\mathbf{v}_t^i$ of an active client $i$ in iteration $t$ is expressed as follows: 
\begin{align}
    \mathbf{v}_t^i &= \nabla f_i(\mathbf{w}_t) - \nabla f_i\left(\mathbf{w}_{t-\tau(t,i)} \right) + \mathbf{v}_{t-\tau(t,i)} \nonumber \\
    &\overset{(\text{a})}{=} \nabla f_i(\mathbf{w}_t) + \left[ - \nabla f_i\left(\mathbf{w}_{t-\tau(t,i)} \right) + \nabla f\left(\mathbf{w}_{t-\tau(t,i)} \right) \right] \nonumber \\
    & + \underbrace{[- \nabla f\left(\mathbf{w}_{t-\tau(t,i)} \right) + \mathbf{v}_{t-\tau(t,i)}]}_{A_1} \nonumber \\
    &= \nabla f_i(\mathbf{w}_t) + A_1 + \Delta_{-i}(\mathbf{w}_{t-\tau(t,i)}) \nonumber \\
    &\overset{(\text{b})}{=} \nabla f_i(\mathbf{w}_t) + A_1 + \underbrace{\Delta_{-i}(\mathbf{w}_{t-\tau(t,i)}) - \Delta_{-i}(\mathbf{w}_{t})}_{A_2} + \Delta_{-i}(\mathbf{w}_{t}) \nonumber \\
    &\overset{(\text{c})}{=} \nabla f(\mathbf{w}_t) + A_1 + A_2. 
    \label{eq:approximate}
\end{align}
In (a) of \eqref{eq:approximate}, the term $A_1$ denotes the difference between $\nabla f\left(\mathbf{w}_{t-\tau(t,i)} \right)$ and $\mathbf{v}_{t-\tau(t,i)}$ in the latest iteration $t-\tau(t,i)$ where client $i$ is active. 
We assume that all clients in the first training iteration are active, which, intuitively, ensures a high-quality approximation in training iteration $t=0$ and leads to small accumulated error in subsequent iterations (i.e., $A_1$) as well.
The term $A_2$ in (b) comes from the model difference in two iterations, i.e., $\mathbf{w}_{t-\tau(t,i)} - \mathbf{w}_{t}$, which is bounded via gradients and can be eliminated with a proper choice of the global learning rate.
In addition, equation (c) follows the fact that $\Delta_{-i}(\mathbf{w}_{t}) + \nabla f_i(\mathbf{w}_t) = \nabla f(\mathbf{w}_t)$.
For the general case, the mini-batch gradients can be viewed as unbiased estimates of the full-batch gradient, while the effect of multiple local steps can also be bounded.
Therefore, the applied update for each client mimics the central update.
The above derivation also shows that $\left\| \mathbf{v}_i^t - \nabla f(\mathbf{w}_t) \right\|_2^2 = \left\| A_1 + A_2 \right\|_2^2$ is relatively stable.
As will be shown in the next section, the proposed MimiC algorithm converges to a stationary point of the global loss function even with arbitrary client dropouts in cross-device FL.

\textbf{Comparison with SCAFFOLD \cite{scaffold}.}
We notice that SCAFFOLD \cite{scaffold}, a seminal algorithm for handling the non-IID issues in FL, also uses correction variables to modify the local updates.
To be specific, SCAFFOLD extends a classic variance reduction technique named SAGA \cite{SAGA} to FL and modifies the client update in the $k$-th local step as follows:
\begin{equation}
    \hat{g}_{i}(\mathbf{w}_{t,k}^{i}) \triangleq g_{i}(\mathbf{w}_{t,k}^{i}) - g_{i}(\mathbf{w}_{t}) + \frac{1}{|\mathcal{S}_t|} \sum_{j\in\mathcal{S}_t} g_{j}(\mathbf{w}_{t}).
    \label{eq:scaffold}
\end{equation}
The local model $\mathbf{w}_{t,k}^{i}$ is then updated using the modified update in \eqref{eq:scaffold}.
After $K$ local updating steps, each client uploads the accumulated model update $\sum_{k=1}^{K} \hat{g}_{i}(\mathbf{w}_{t,k}^{i})$ to the server for aggregation.
In each training iteration, by correcting the local gradient $g_{i}(\mathbf{w}_{t,k}^{i})$ using the update deviation evaluated with the initial model $\mathbf{w}_{t}$ (i.e., $\frac{1}{|\mathcal{S}_t|} \sum_{j\in\mathcal{S}_t} g_{j}(\mathbf{w}_{t}) - g_{i}(\mathbf{w}_{t})$), the modified update in \eqref{eq:scaffold} alleviates the negative effect of biased local updates.
However, the theoretical convergence guarantee of SCAFFOLD requires $\mathbb{E} \left[\frac{1}{|\mathcal{S}_t|} \sum_{j\in\mathcal{S}_t} g_{j}(\mathbf{w}_{t}) \right] = \frac{1}{N} \sum_{j\in\mathcal{N}} g_{j}(\mathbf{w}_{t})$ and only holds with uniform client participation.
In comparison, MimiC aims to correct the training bias caused by both client dropout and data heterogeneity. 
Such a bias is encoded in the difference between the local update and the sum of model updates from all clients (i.e., $\frac{1}{N} \sum_{j\in\mathcal{N}} \nabla f_j(\mathbf{w}_t) - \hat{g}_i(\mathbf{w}_{t})$), which, however, is not accessible due to client dropouts.
Hence, we propose to use $\mathbf{c}_{t-\tau(t, i)}^{i} = \mathbf{v}_{t-\tau(t, i)} - \hat{g}_i(\mathbf{w}_{t-\tau(t, i)})$ from iteration $t-\tau(t, i)$ where client $i$ was last active to correct the local update. 
We demonstrate the superior performance of MimiC compared with SCAFFOLD in Section \ref{sec:evaluation}.

\textbf{Practical implementations.}
If there are a large number of clients, the proposed MimiC algorithm can be implemented in an alternative way to amortize the computation and memory overhead at the server, where the clients correct their gradient updates locally and upload the modified gradients. To be specific, each client stores a local correction variable $\mathbf{c}_{t^\prime}^i = \mathbf{v}_{t^\prime} - \hat{g}_i(\mathbf{w}_{t^\prime})$ after an active iteration $t^\prime$. In the next active iteration $t$, client $i$ computes gradient $\hat{g}_i(\mathbf{w}_{t})$ and upload the corrected update $\mathbf{v}_{t}^{i} = \hat{g}_i(\mathbf{w}_{t}) + \mathbf{c}_{{t^\prime}}^i$ to the server for aggregation. In this way, compared with FedAvg, each client only needs to store one more variable $\mathbf{c}_{{t^\prime}}^i$ with the same dimension as the model gradient and there is no additional computation and memory requirement at the server.
Notably, as the server only requires a sum $\sum_{i\in \mathcal{S}_t} \mathbf{v}_{t}^{i}$, secure aggregation techniques \cite{elkordy2022heterosag,bonawitz2017practical} can be further implemented to preserve the privacy of local models.

\section{Convergence of MimiC}\label{sec:convergenceofmimic}
In this section, we analyze the convergence performance of MimiC.
Similar to the analysis in Section \ref{sec:fedavg}, we first derive an upper bound of the gradient estimation error $\mathcal{E}_t \triangleq \mathbb{E} [\left\| \mathbb{E}[\mathbf{v}_t] - \nabla f (\mathbf{w}_t) \right\|^2 ]$ in the following lemma.

\begin{lemma}\label{lem:E_t_2}
If the global learning rates satisfy:
\begin{equation}
    \rho_t \triangleq \frac{\eta_t|\mathcal{S}_{t+1}|}{\eta_{t+1}|\mathcal{S}_t|} >1, \forall t \in [T],
\label{eq:lr-1}
\end{equation}
there exists an arbitrarily small positive constant $\nu>0$ such that the gradient estimation error $\mathcal{E}_t^{\mathrm{M}}$ in MimiC can be upper bounded as follows:
\begin{align}
    \mathcal{E}_t^{\mathrm{M}}
    \leq & \frac{1}{|\mathcal{S}_t|} \sum_{i\in\mathcal{S}_t} (1+\alpha_t) \mathcal{E}_{t-\tau(t,i)} \nonumber \\
    + & 2 \left(1+\frac{1}{\alpha_t} \right) \varphi_K \tau(t,i) \sum_{s=1}^{\tau(t,i)}\eta_{t-s}^2  \frac{\sigma^2}{|\mathcal{S}_{t-s}|K} \nonumber \\
    + & 2 \left(1+\frac{1}{\alpha_t} \right) \varphi_K \tau(t,i) \sum_{s=1}^{\tau(t,i)}\eta_{t-s}^2 \left\| \mathbb{E}[\mathbf{v}_{t-s}] \right\|^2,
    \label{eq:e_t_2}
\end{align}
where $\varphi_K \triangleq \left(\frac{(2+2\eta_L^2L^2)^K - 1}{K(2\eta_L^2L^2 + 1)} + 1 \right) L^2$ and $\alpha_t \triangleq \frac{\eta_t|\mathcal{S}_{t+1}|}{\eta_{t+1}|\mathcal{S}_t|} - 1 - \nu$.
\end{lemma}
\begin{proof}[Proof Sketch]
We first utilize the Cauchy–Schwarz inequality to show
\begin{align}
    \mathcal{E}_t^{\mathrm{M}}
    \leq \frac{1}{|\mathcal{S}_t|} \! \sum_{i\in\mathcal{S}_t} \mathbb{E} \left[\left\| \mathbb{E}[\mathbf{v}_t^i] - \nabla f (\mathbf{w}_t) \right\|^2 \right].
\end{align}
Afterwards, we respectively upper bound the estimation error between $\mathbb{E}[\mathbf{v}_t^i]$ and $\nabla f(\mathbf{w}_t)$ for each client in $\mathcal{S}_{t}$ and sum these error bounds up.
The main challenge is that $\mathbb{E}[\mathbf{v}_t^i]$ and $\nabla f(\mathbf{w}_t)$ are not directly related.
To connect these terms, we resort to the intermediate terms $\nabla f(\mathbf{w}_{t-\tau(t,i)})$ and $\mathbb{E}[\hat{g}_i(\mathbf{w}_{t-\tau(t,i)})]$.
We then use the fact that $\|\mathbf{x}+\mathbf{y} \|^2 \leq (1+\alpha) \|\mathbf{x}\|^2 + (1+\frac{1}{\alpha}) \|\mathbf{y}\|^2, \forall \alpha >0$ with $\alpha=\alpha_t$ and $\alpha=1$ sequentially.
Note that to make $\alpha_t > 0$, the condition in \eqref{eq:lr-1} is required, which can be satisfied by properly adapting the global learning rates.
Thereafter, an upper bound of $\mathcal{E}_t^{\mathrm{M}}$ can be decomposed as follows:
\begin{align}
    \mathcal{E}_t^{\mathrm{M}}
    \leq & \frac{1}{|\mathcal{S}_t|} \sum_{i\in\mathcal{S}_t} (1+\alpha_t) \underbrace{ \mathbb{E}\left[\left\| \mathbb{E}[\mathbf{v}_{t-\tau(t,i)}] - \nabla f(\mathbf{w}_{t-\tau(t,i)}) \right\|^2\right]}_{\mathcal{E}_{t,1}^{i}} \nonumber \\
    & + 2 \left(1+\frac{1}{\alpha_t} \right) \underbrace{\mathbb{E}\left[\left\| \mathbb{E}[\hat{g}_i(\mathbf{w}_t)] - \mathbb{E}[\hat{g}_i(\mathbf{w}_{t-\tau(t,i)})] \right\|^2\right]}_{\mathcal{E}_{t,2}^{i}} \nonumber \\
    & + 2\left(1+\frac{1}{\alpha_t} \right) \underbrace{ \mathbb{E}\left[\left\| \nabla f(\mathbf{w}_{t-\tau(t,i)}) - \nabla f(\mathbf{w}_t) \right\|^2\right]}_{\mathcal{E}_{t,3}^{i}}.
    \label{eq:decomposition}
\end{align}
In \eqref{eq:decomposition}, $\mathcal{E}_{t,1}^{i}$ is the gradient estimation error $\mathcal{E}_{t-\tau(t,i)}$ in iteration $t-\tau(t,i)$ where client $i$ was last active.
Besides, both $\mathcal{E}_{t,2}^{i}$ and $\mathcal{E}_{t,3}^{i}$ depend on the difference between $\mathbf{w}_t$ and $\mathbf{w}_{t-\tau(t,i)}$, which can be upper bounded by the aggregated updates from iteration $t-\tau(t,i)$ to $t$, and rearranging terms related to $\sigma^2$ and $\eta_{t-s}^2 \left\| \mathbb{E}[\mathbf{v}_{t-s}] \right\|^2$, respectively.
Please refer to Appendix \ref{proof:lem:E_t_2} for the detailed proof.
\end{proof}

\begin{remark}
As will be shown in Theorem \ref{thm:mimic}, the terms on the RHS of \eqref{eq:e_t_2} can be eliminated by using proper learning rates.
By comparing the upper bounds of $\mathcal{E}_t^{\mathrm{F}}$ in Lemma \ref{lem:E_t_1} and $\mathcal{E}_t^{\mathrm{M}}$ in Lemma \ref{lem:E_t_2}, it can be noticed that the term $\gamma_t$ caused by arbitrary client dropouts in FedAvg is eliminated in MimiC.
This is because the correction variables $\{ \mathbf{c}_t^i \}$'s enforce the applied updates $\{ \mathbf{v}_t^i \}$'s to mimic the central update $\nabla f(\mathbf{w}_t)$ with a bounded divergence.
With a proper choice of the learning rates, the corresponding objective mismatch can also be eliminated.
\end{remark}

Similar to Lemma \ref{lem:phi_t_1}, we derive an upper bound for $\Phi_t^{\mathrm{M}}$ in MimiC in the following lemma.
\begin{lemma}\label{lem:phi_t_2}
In MimiC, we have:
\begin{equation}
    \Phi_t^{\mathrm{M}} \leq \frac{\sigma^2}{|\mathcal{S}_t|K}. 
    \label{eq:mimic:phi}
\end{equation}
\end{lemma}
\begin{proof}
Following a similar proof in Lemma \ref{lem:phi_t_1}, we have:
\begin{align}
    \Phi_t^{\mathrm{M}} 
    = & \frac{1}{|\mathcal{S}_t|^2} \sum_{i\in\mathcal{S}_t} \mathbb{E} \left[\left\| \mathbf{v}_t^i - \mathbb{E} [\mathbf{v}_t^i] \right\|^2 \right] \nonumber \\
    \overset{(\text{a})}{=} & \frac{1}{|\mathcal{S}_t|^2} \sum_{i\in\mathcal{S}_t} \mathbb{E} \left[\left\| \left( \hat{g}_i(\mathbf{w}_t) \right) - \mathbb{E}\left[ \hat{g}_i(\mathbf{w}_t) \right] \right\|^2 \right] \nonumber \\
    \leq & \frac{1}{|\mathcal{S}_t|K} \sigma^2,
\end{align}
where (a) holds since the expectation is taken with respect to the mini-batch sampling in iteration $t$, which is irrelevant to previous updates.
\end{proof}

In order to show the convergence of MimiC, we present an upper bound for a weighted sum of the expected gradient norm in the following theorem.

\begin{theorem}\label{thm:mimic}
With Assumptions \ref{ass:clients}-\ref{ass-2}, if learning rates $\{\eta_t\}_{t=0}^{T}$ satisfy $\rho_t > 1$ and
\begin{equation}
    \frac{1}{\eta_{t}} \left( \frac{1}{2\eta_t} - \frac{L}{2} \right) \geq \frac{(\rho_t-\nu)\varphi_K \tau_{\max} N}{2\nu|\mathcal{S}_{t}|},
    \label{eq:lr-2}
\end{equation}
we have:
\begin{align}
    & \frac{1}{\sum_{t=0}^{T-1} \eta_t} \sum_{t=0}^{T-1} \eta_t \mathbb{E} [\left\| \nabla f(\mathbf{w}_t) \right\|^2] \nonumber \\
    \leq & \frac{ 2 \Delta }{\sum_{t=0}^{T-1} \eta_t}
    + \frac{2}{\sum_{t=0}^{T-1} \eta_t} \sum_{t=0}^{T-1} \frac{\sigma^2}{|\mathcal{S}_t|K} \left[ \frac{\eta_t^2 L}{2} \right. \nonumber \\
    + & \left. \frac{\rho_t(\rho_t-\nu) \eta_t \varphi_K \tau_{\max}^2}{\nu(\rho_t-1-\nu)} \frac{\eta_{t+1}^2}{|\mathcal{S}_{t+1}|} \right].
    \label{eq:thm-mimic}
\end{align}
\end{theorem}
\begin{proof}[Proof Sketch]
Recall the one-round progress in \eqref{eq:one-round}, the following inequality holds:
\begin{align}
    &\mathbb{E} [f(\mathbf{w}_{t+1})] + \left(\frac{\eta_t}{2} - \frac{\eta_t^2 L}{2} \right) \left\| \mathbb{E}[\mathbf{v}_t] \right\|^2 \nonumber \\
    \leq & \mathbb{E} [f(\mathbf{w}_{t})] - \frac{\eta_t}{2} \mathbb{E} \left[ \left\| \nabla f(\mathbf{w}_t) \right\|^2 \right]+ \frac{\eta_t^2 L}{2|\mathcal{S}_t|K} \sigma^2 + \frac{\eta_t}{2} \mathcal{E}_t^{\mathrm{M}}, \label{eq:one-round-mimic-copy}
\end{align}
where we apply the result in Lemma \ref{lem:phi_t_2}.
To proceed, we add $\frac{\rho_t \eta_{t}}{2\nu} \mathcal{E}_t^{\mathrm{M}}$ at both sides of \eqref{eq:one-round-mimic-copy}, apply the results in Lemma \ref{lem:E_t_2}, and sum up the results over $t=0,1,\dots,T-1$ as follows:
\begin{align}
    & \sum_{t=0}^{T-1} \mathbb{E} [f(\mathbf{w}_{t+1})] + \sum_{t=0}^{T-1} \left(\frac{\eta_t}{2} - \frac{\eta_t^2 L}{2} \right) \left\| \mathbb{E}[\mathbf{v}_t] \right\|^2 + \sum_{t=0}^{T-1} \frac{\rho_t\eta_{t}}{2\nu} \mathcal{E}_t^{\mathrm{M}} \nonumber \\
    \leq & \sum_{t=0}^{T-1} \mathbb{E} [f(\mathbf{w}_{t})] - \sum_{t=0}^{T-1} \frac{\eta_t}{2} \mathbb{E} \left[ \left\| \nabla f(\mathbf{w}_t) \right\|^2 \right]+ \sum_{t=0}^{T-1} \frac{\eta_t^2 L}{2|\mathcal{S}_t|K} \sigma^2 \nonumber \\
    + & \sum_{t=0}^{T-1} \left[ \frac{\eta_{t}}{2} +  \frac{\rho_t (\rho_t-\nu)}{2\nu} \sum_{i\in\mathcal{S}_{t}} \sum_{s=1}^{\Tilde{\tau}(t)} \frac{\eta_{t+s}}{|\mathcal{S}_{t+s}|} \mathbf{1}\{ \tau(t+s,i)=s \} \right] \mathcal{E}_t^{\mathrm{M}} \nonumber \\
    + & \sum_{t=0}^{T-1} \frac{\rho_t(\rho_t-\nu)}{\nu(\rho_t-1-\nu)}  \frac{\eta_t \varphi_K}{|\mathcal{S}_t|K} \sum_{i\in\mathcal{S}_t}
    \tau(t,i) \sum_{s=1}^{\tau(t,i)} \frac{\eta_{t-s}^2 \sigma^2}{|\mathcal{S}_{t-s}|} \nonumber \\
    + & \sum_{t=0}^{T-1} \frac{\rho_t(\rho_t-\nu)}{\nu(\rho_t-1-\nu)} \frac{\eta_t \varphi_K}{|\mathcal{S}_t|} \sum_{i\in\mathcal{S}_t} \tau(t,i) \sum_{s=1}^{\tau(t,i)}\eta_{t-s}^2 \left\| \mathbb{E}[\mathbf{v}_{t-s}] \right\|^2. \label{eq:help-1-copy}
\end{align}
To verity the inequality in \eqref{eq:thm-mimic}, we further utilize the property of $\rho_t > 1$.
Please refer to Appendix \ref{proof:thm:mimic} for the detailed proof.
\end{proof}

Before presenting the convergence result of MimiC, we show the existence of the global learning rates specified in Theorem \ref{thm:mimic} via the following lemma.
\begin{lemma}\label{lem:lr}
    There exists a set of global learning rates $\{\eta_t\}_{t=0}^{T-1}$ that satisfies the conditions in Theorem \ref{thm:mimic} (i.e., \eqref{eq:lr-1} and \eqref{eq:lr-2}), and $\lim_{T\rightarrow \infty} \sum_{t=0}^{T-1} \eta_t = \infty$, $\lim_{T\rightarrow \infty} \sum_{t=0}^{T-1} \eta_t^2 < \infty$.
\end{lemma}
\begin{proof}[Proof Sketch]
    The proof can be obtained by constructing $\eta_{t} = \frac{c|\mathcal{S}_{t}|}{t+\beta}$, where $c$ and $\beta$ are positive constants. We show that there exists $c>0$ such that all the assumptions of $\{\eta_{t}\}_{t=0}^{T-1}$ can be simultaneously satisfied.
    Please refer to Appendix \ref{proof:lr} for the detailed proof.
\end{proof}

With the results in Theorem \ref{thm:mimic} and Lemma \ref{lem:lr}, we show the convergence of MimiC in the following corollary, and the proof is deferred to Appendix \ref{proof:Corollary1}.

\begin{corollary}\label{corollary1}
If the learning rates satisfy $\eta_L = \mathcal{O} \left(\frac{1}{L\sqrt{T}} \right)$, the conditions in Theorem \ref{thm:mimic}, and $\lim_{T\rightarrow \infty} \sum_{t=0}^{T-1} \eta_t = \infty$, $\lim_{T\rightarrow \infty} \sum_{t=0}^{T-1} \eta_t^2 < \infty$, the RHS of \eqref{eq:thm-mimic} converges to zero as $T\rightarrow \infty$, i.e., MimiC converges to a stationary point of the global loss function.
\end{corollary}

The above results show the superiority of MimiC over FedAvg in the presence of random client dropouts.
By correcting the local updates, MimiC can ensure convergence, while FedAvg may only arrive at a solution around a stationary point of the global loss function.
While the convergence of MimiC does not require a certain regime of $\tau_{\text{max}}$, a large value does have a negative impact on the convergence behavior. In particular, the learning rate needs to be reduced to ensure model convergence, and the convergence time may be prolonged significantly.

In some cases, clients become active with certain probabilities in each training iteration instead of having bounded dropout iterations. 
The active probability of a client may be determined by various factors, e.g., the communication link quality and local computation workloads.
For example, clients such as Unmanned Aerial Vehicles (UAVs) can fly away from the server at any training iteration \cite{nguyen2021federated}, and the change of path loss results in different active probabilities.
In the following, we show the convergence of MimiC under probabilistic client dropout as specified in Assumption \ref{ass:client-prob}.

\begin{assumption}\label{ass:client-prob}
    Client $i\in\mathcal{N}$ becomes active with probability $p_t^i \in(0,1]$ in the $t$-th training iteration.
\end{assumption}

\begin{theorem}\label{high-prob}
With Assumptions \ref{ass-1}-\ref{ass-2} and \ref{ass:client-prob}, if learning rates $\{\eta_t\}_{t=0}^{T}$ satisfy $\rho_t > 1$ and $\frac{1}{\eta_{t}} \left( \frac{1}{2\eta_t} - \frac{L}{2} \right) \geq \frac{(\rho_t-\nu)\varphi_K \tilde{\tau}(t) N}{2\nu|\mathcal{S}_{t}|}$, then for arbitrary $\delta \in (0,1)$, we have:
\begin{align}
    & \frac{1}{\sum_{t=0}^{T-1} \eta_t} \sum_{t=0}^{T-1} \eta_t \mathbb{E} [\left\| \nabla f(\mathbf{w}_t) \right\|^2] \nonumber \\
    \leq & \frac{ 2 \Delta }{\sum_{t=0}^{T-1} \eta_t}
    + \frac{2}{\sum_{t=0}^{T-1} \eta_t} \sum_{t=0}^{T-1} \frac{\sigma^2}{|\mathcal{S}_t|K} \left[ \frac{\eta_t^2 L}{2} \right. \nonumber \\
    + & \left.\frac{\rho_t(\rho_t-\nu) \eta_t \varphi_K}{\nu(\rho_t-1-\nu) (\min_{i}p_t^i)^2} \left( \log \left( \frac{Nt}{\delta} \right)+1 \right)^2 \frac{\eta_{t+1}^2}{|\mathcal{S}_{t+1}|} \right], 
    \label{eq:thm-mimic-2}
\end{align}
with probability at least $1 - \delta$, where $\tilde{\tau}(t) \triangleq \max_{i\in \mathcal{S}_t} \tau(t,i)$.
\end{theorem}

\begin{proof}[Proof Sketch]
We obtain the result with a similar proof of Theorem \ref{thm:mimic}, where the gradient estimation error $\mathcal{E}_t^{\mathrm{M}}$ is expressed in terms of $\Tilde{\tau}(t)$.
We then show $\Tilde{\tau}(t)$ can be upper bounded using $\{p_i^t\}$'s with a high probability.
Please refer to Appendix \ref{proof:high-prob} for the detailed proof.
\end{proof}

The following corollary shows the convergence result of MimiC under probabilistic client dropout, and the proof is deferred to Appendix \ref{proof:Corollary2}.

\begin{corollary} \label{Corollary:high-prob}
If the learning rates satisfy $\eta_L = \mathcal{O} \left(\frac{1}{L\sqrt{T}} \right)$, the conditions in Theorem \ref{thm:mimic}, and $\lim_{T\rightarrow \infty} \sum_{t=0}^{T-1} \eta_t = \infty$, $\lim_{T\rightarrow \infty} \sum_{t=0}^{T-1} \eta_t^2 < \infty$, the RHS of \eqref{eq:thm-mimic-2} converges to zero as $T\rightarrow \infty$.
This guarantees that MimiC converges to a stationary point of the global loss function with a high probability under probabilistic client dropout.
\end{corollary}

The conventional FL algorithms are hard to ensure convergence in the critical scenario with arbitrary client dropouts.
To preserve model convergence in this scenario, existing methods \cite{mifa,yan2020intermittent} utilize the memorized latest updates as a substitute for those dropout clients, which suffers from the staleness of local updates.
In contrast, our proposed method MimiC corrects the local updates by mimicking an imaginary central update.
By doing so, we can achieve substantial training performance improvement as verified by the simulation results in Section \ref{sec:evaluation}, in additional to the theoretical convergence guarantee.

\section{Performance Evaluation}\label{sec:evaluation}

\subsection{Settings}
We simulate a cross-device FL system with one central server and $N=30$ clients.
To comprehensively demonstrate the effects of client dropouts, we simulate three scenarios with different client availability patterns, including:
\begin{enumerate}
    \item Clients have bounded consecutive dropout iterations \cite{yan2020intermittent}: Each client becomes active every $\tau_{\text{max}}(i)$ iterations in a round-robin fashion, where $\tau_{\text{max}}(i)$ is uniformly sampled from $[0, \tau_{\text{max}}]$.

    \item Clients have static active probabilities \cite{mifa}: All clients have the same and fixed active probability. 
    In iteration $t$, a random number $r_t^i \in [0, 1]$ is uniformly generated for the $i$-th client. The client is active if $r_t^i$ is no larger than its active probability; otherwise it drops out of training in this iteration.
    
    \item Clients have time-varying active probabilities \cite{bonawitz2019towards,huang2022contextfl}:
    In iteration $t$, each client randomly generates a weight $a_{t}^{i} \in [1, 10]$, of which, a large value implies higher likelihood to become active. We determine $P\times N$ active clients in simulations by sampling clients without replacement using probability vector $[a_{t}^{1},\cdots,a_{t}^{N}] \slash \sum_{j\in\mathcal{N}} a_{t}^{j}$, where $P\in\left[0,1\right]$ is a predefined ratio.
\end{enumerate}

For comparisons, we adopt FedAvg \cite{fedavg}, FedProx \cite{fedprox}, SCAFFOLD \cite{scaffold}, and MIFA \cite{mifa} as baseline algorithms.
In MIFA, the server stores the latest model update for each client and uses it for aggregation in case of client dropout. Other algorithms are benchmark FL training algorithms but not specialized for combating client dropouts. 
It is also worth noting that SCAFFOLD requires each client to upload two updated variables and thus incurs double communication cost in each training iteration compared with other methods.
Hence, we evaluate the accuracy of different algorithms on the test dataset with respect to the same amount of communication costs (i.e., the number of model uploadings) for fair comparisons.
Table \ref{table:compare} summarizes the main characteristics of these algorithms.

\begin{table*}[t]
\caption{Comparisons of different algorithms.}\label{table:compare}
\resizebox{\textwidth}{!}{
\begin{tabular}{|c|c|c|c|c|c|}
\hline
Algorithm            & \begin{tabular}[c]{@{}c@{}}Mitigating \\ Non-IID Issues\end{tabular} & \begin{tabular}[c]{@{}c@{}}Mitigating \\ Client Dropout\end{tabular} & \begin{tabular}[c]{@{}c@{}}Communication \\ Overhead Per Round\end{tabular} & \begin{tabular}[c]{@{}c@{}}Convergence Requirements \\ under Client Dropouts \end{tabular} & \begin{tabular}[c]{@{}c@{}}Potential for \\ Secure Aggregation\end{tabular} \\ \hline
FedAvg \cite{fedavg}    &  \twemoji{multiply} &  \twemoji{multiply}                         & $\mathcal{O}(M)$                           &  Uniform client dropout                    &  \twemoji{check mark}           \\ \hline
FedProx \cite{fedprox}  &  \twemoji{check mark}  &  \twemoji{multiply}                       &  $\mathcal{O}(M)$                          &  Uniform client dropout                         & \twemoji{check mark}       \\ \hline
SCAFFOLD \cite{scaffold} &   \twemoji{check mark} &  \twemoji{multiply}                        & $\mathcal{O}(2M)$                           & Uniform client dropout                          &  \twemoji{check mark}      \\ \hline
MIFA \cite{mifa}     &   \twemoji{check mark}   & \twemoji{check mark}                      &  $\mathcal{O}(M)$                          &  \begin{tabular}[c]{@{}l@{}} 
      $\bullet$ Hessian lipschitz of loss functions; \\
      $\bullet$ Bounded consecutive dropout iterations, or \\
      $\bullet$ Probabilistic client dropout
\end{tabular}       & \twemoji{multiply}      \\ \hline
MimiC (Ours)      &   \twemoji{check mark}  & \twemoji{check mark}                      &       $\mathcal{O}(M)$                     &   \begin{tabular}[c]{@{}l@{}} 
      $\bullet$ Bounded consecutive dropout iterations, or \\
      $\bullet$ Probabilistic client dropout
\end{tabular}                        & \twemoji{check mark}       \\ \hline
\end{tabular}
}
\end{table*}

We evaluate these algorithms on two benchmark image datasets, including the FMNIST \cite{fmnist} and CIFAR-10 \cite{cifar10} datasets.
Both datasets have ten classes of labels.
The training datasets of FMNIST and CIFAR-10 respectively have 60,000 and 50,000 data samples in total, which are assigned to 30 clients in equal size.
To simulate the non-IID setting, we divide the training dataset into 60 shards, each of which includes one random class of data samples, and assign two random shards to each client.
In the FMNIST training task, we adopt a convolutional neural network (CNN) model \cite{li2021silo}, which has two $5\times 5$ convolutional layers followed by $2\times 2$ max pooling layers and two fully connected layers with ReLU activation, while a VGG-11 model \cite{vgg} is adopted in the CIFAR-10 training task.
The detailed experimental setup are summarized in Table \ref{table-setup}.
Under the above setup, FedAvg without client dropout can achieve accuracy of $76.82\%$ and $62.48\%$ after 200 and 400 model uploadings on the FMNIST and CIFAR10 datasets, respectively.

\begin{table}[ht]
\caption{Experimental setup.}
\label{table-setup}
\centering
\begin{tabular}{|c|c|c|}
\hline
Parameter                           & FMNIST & CIFAR-10 \\
\hline
\# of classes/client                 & 2      & 2     \\ \hline
\# of data samples/client            & 2,000   & 1,666   \\ \hline
Batch size                          & 16     & 32     \\ \hline
Local epochs ($K$)                   & 5      & 5     \\ \hline
Local learning rate $\eta_L$ & 0.01   & 0.01    \\ \hline
Learning rate $\eta_t$ & $0.01\times 0.95^{t-1}$   & $0.01\times 0.99^{t-1}$    \\ \hline
\# of trials & 3 & 3 \\ \hline
\end{tabular}
\end{table}

\subsection{Results}

\begin{figure}[ht]
    \centering
    \includegraphics[width=0.7\linewidth]{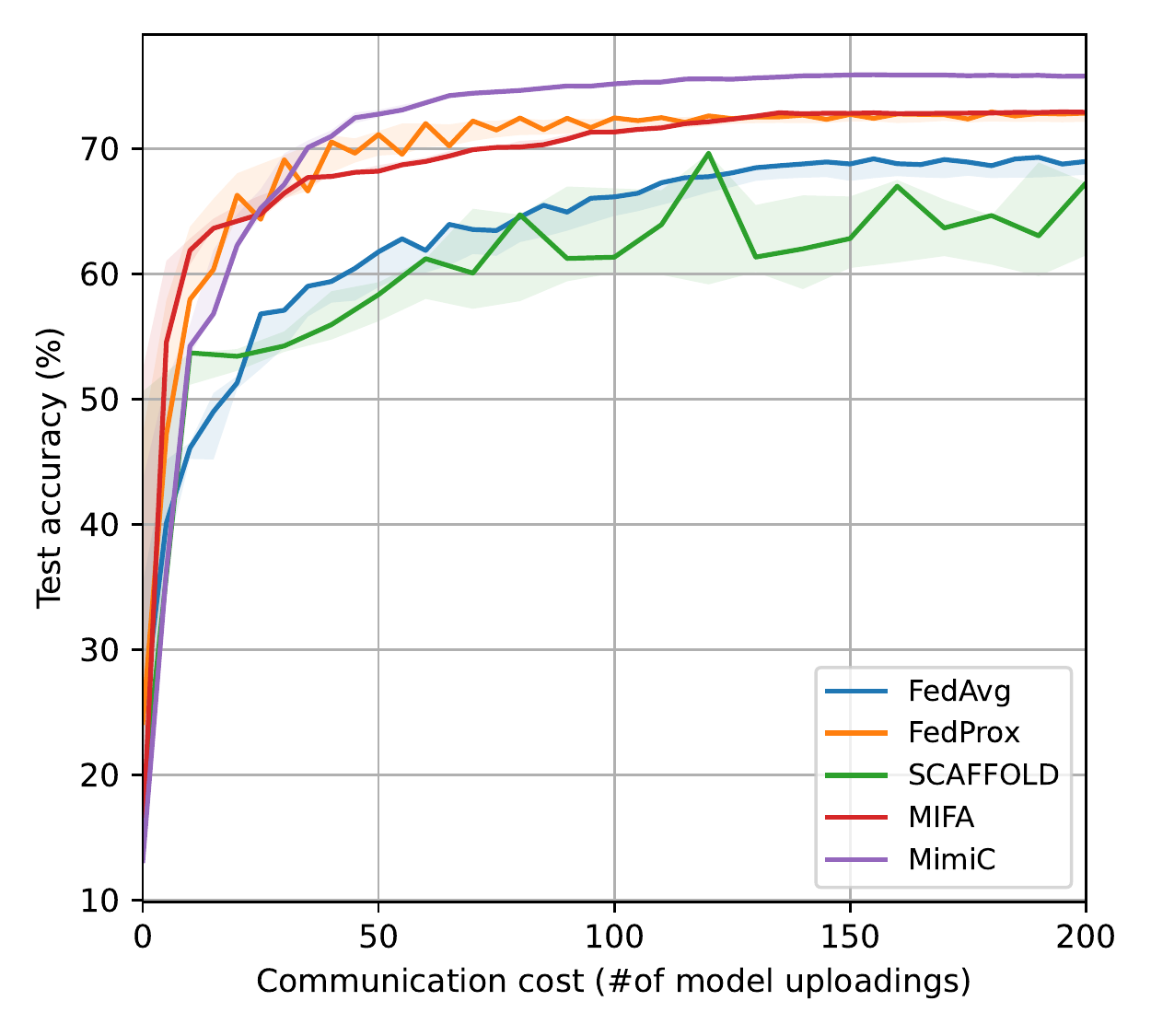}
    \caption{Test accuracy on the FMNIST dataset, where client $i$ becomes active every $\tau_{\text{max}}(i)$ iterations ($\tau_{\text{max}} = 20$).}
    \label{fig:tau-f}
\end{figure}

\begin{figure}[ht] 
    \centering
    \includegraphics[width=0.7\linewidth]{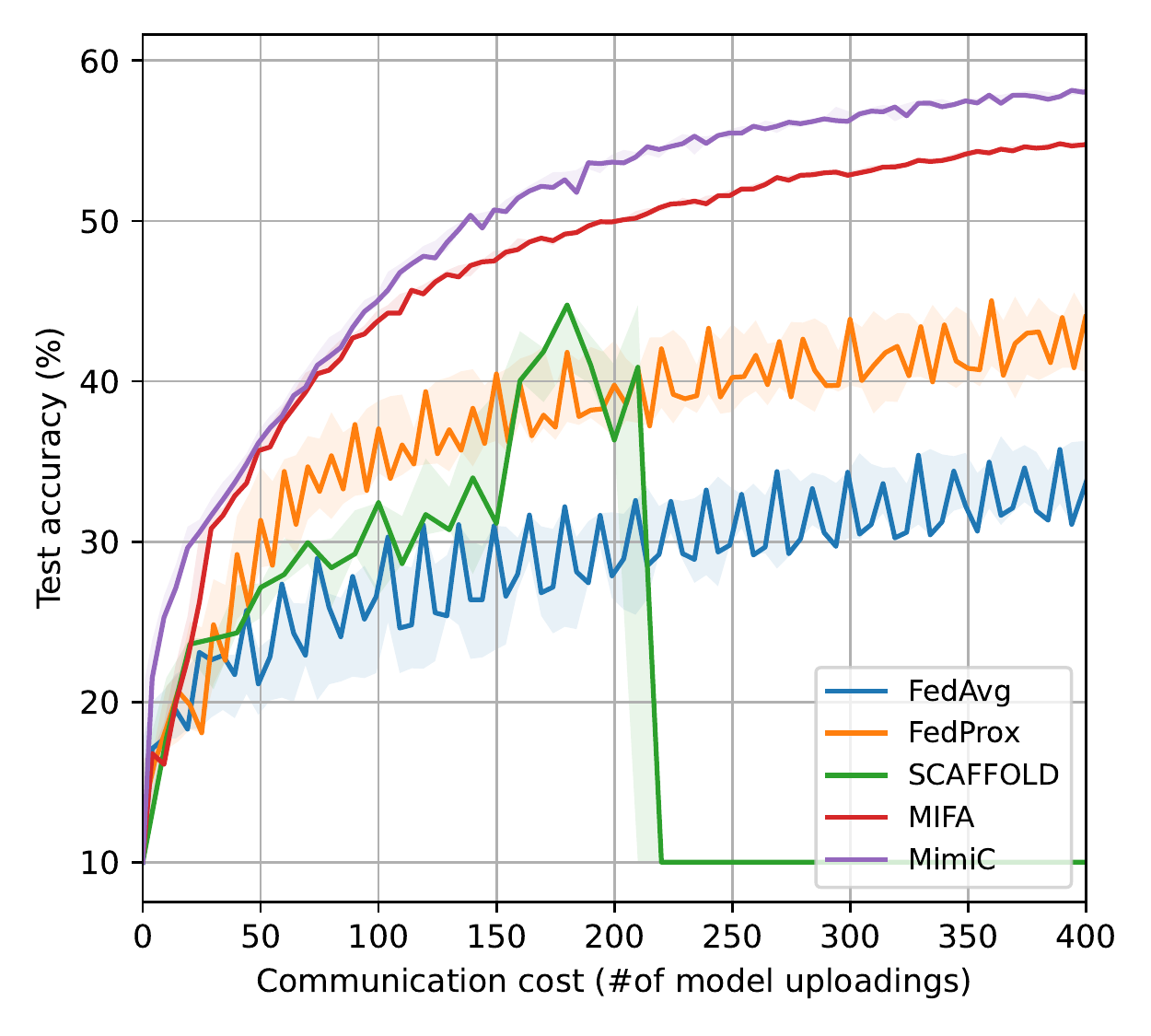}
    \caption{Test accuracy on the CIFAR-10 dataset, where client $i$ becomes active every $\tau_{\text{max}}(i)$ iterations ($\tau_{\text{max}} = 20$).}
    \label{fig:tau-c}
\end{figure}

\begin{figure}[!t]
    \centering
    \includegraphics[width=0.7\linewidth]{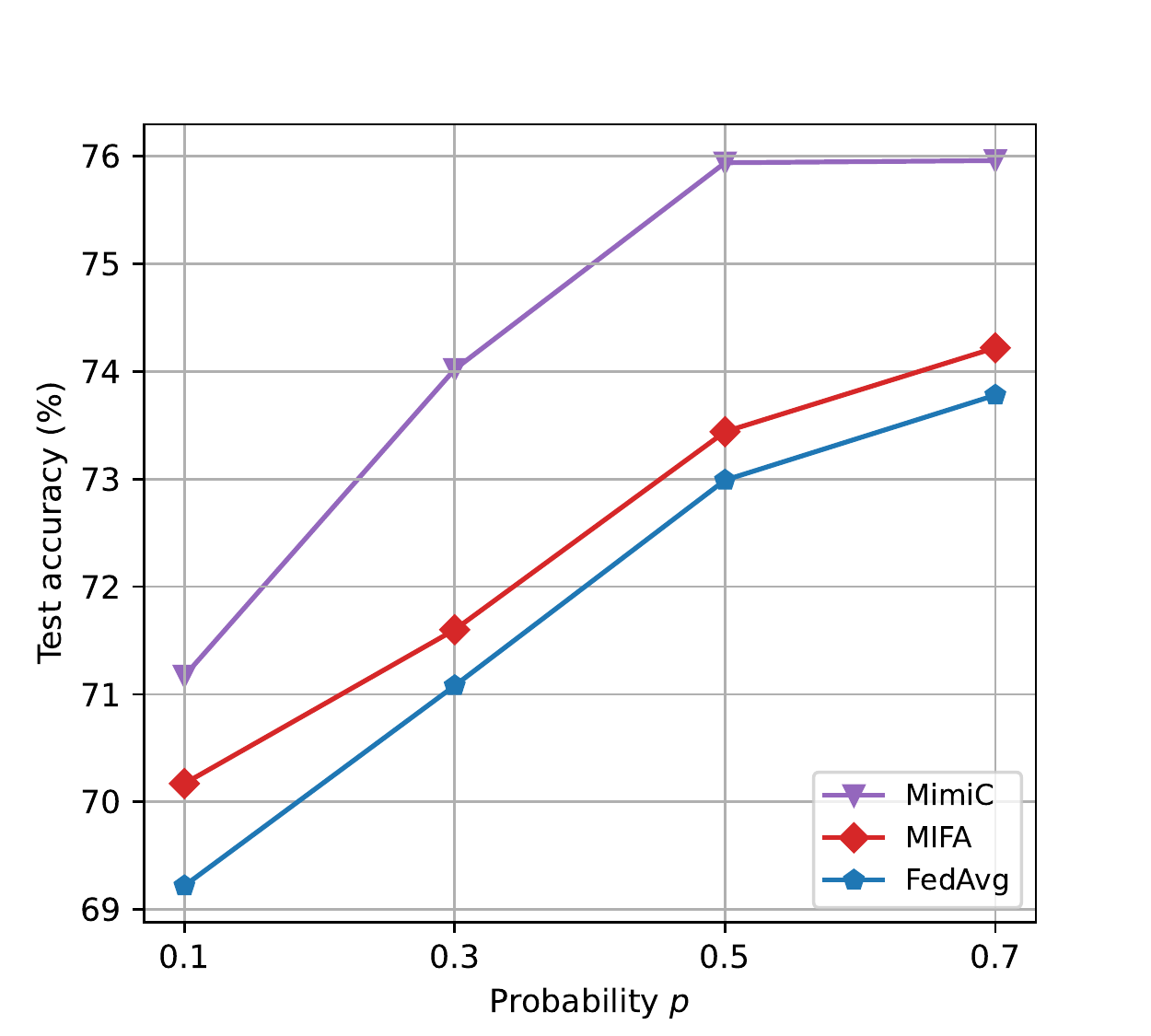}
    \caption{Test accuracy vs. the active probability of different algorithms on the FMNIST dataset.}
    \label{fig:tradeoff}
\end{figure}

We first consider the scenario where clients have bounded consecutive dropout iterations, i.e., $\tau(t,i) \leq \tau_{\text{max}}$.
We assume $\tau_{\text{max}}=20$ and show the test accuracy of different algorithms on the FMNIST and CIFAR-10 datasets in Fig. \ref{fig:tau-f} and Fig. \ref{fig:tau-c}, respectively.
We observe that the test accuracy of FedAvg, FedProx, and SCAFFOLD suffer from different degrees of oscillation over training iterations due to drastic alternation of active clients, which is more obvious in the harder CIFAR-10 training task.
This is because these algorithms lack a dedicated mechanism to deal with random client dropouts.
In contrast, MimiC and MIFA avoid this issue via different client dropout mitigating strategies and thus converge to better global models.
Compared with MIFA, MimiC achieves both improved model accuracy and training efficiency since it effectively corrects the received updates to mimic an imaginary central update.

We also vary $\tau_{\text{max}}$ and show the test accuracy of different algorithms on the FMNIST dataset after 200 model uploadings in Table \ref{table:tau}. It is seen that as the value of $\tau_{\text{max}}$ increases, the test accuracy of all FL algorithms decreases due to more severe client dropouts.
Nevertheless, both MIFA and MimiC suffer from comparatively less performance degradation, again validating their effectiveness in combating client dropouts.
Notably, MimiC consistently outperforms all the baselines with the given communication cost.

\begin{table}[ht]
\centering
\caption{Test accuracy after 200 model uploadings on the FMNIST dataset.}
\label{table:tau}
\begin{tabular}{|c|c|c|c|}
\hline
$\tau_{\text{max}}$ & $20$ & $50$ & $100$ \\ \hline
FedAvg              &   $69.39$   & $68.42$     & $68.21$    \\ \hline
FedProx             & $72.67$     & 67.44     & $64.74$     \\ \hline
SCAFFOLD            & $73.46$     & $72.30$      & $68.7$     \\ \hline
MIFA                & $72.92$     & $72.82$     & $72.20$     \\ \hline
MimiC               & $\mathbf{75.89}$     & $\mathbf{74.86}$     & $\mathbf{74.71}$     \\ \hline
\end{tabular}
\end{table}

Moreover, we evaluate the performance of MimiC, MIFA, and FedAvg on the FMNIST training task by varying the static active probability $p$.
As shown in Fig. \ref{fig:tradeoff}, when clients have larger active probabilities, the test accuracy of the learned model increases as more informative gradients are collected.
However, uploading more gradients means more communication costs, verifying the trade-off between the model accuracy and communication cost.
In all cases, MimiC consistently outperforms the baselines since it mitigates the negative impacts of client dropout via effective model update correction.

Next, we investigate the scenario where clients become active in each iteration with a static probability. 
In particular, the active probability is set as $0.1$ and $0.3$ for the FMNIST and CIFAR-10 training task, as respectively shown in Figs. \ref{fig:ind-f} and \ref{fig:ind-c}.
Compared with the first scenario, the assumption of client availability is much relaxed. However, most FL algorithms still show unstable convergence behavior because of low client availability.
Similar to the observations in Figs. \ref{fig:tau-f} and \ref{fig:tau-c}, both MIFA and MimiC effectively mitigate the negative effects of client dropout, and MimiC converges to a more accurate model with a faster speed.
In addition, we see that SCAFFOLD may fail to converge in the difficult CIFAR-10 training task, as shown in Figs. \ref{fig:tau-c} and \ref{fig:ind-c}, since it suffers from the accumulated gradient estimation error caused by the client dropout.

\begin{figure}[t]
    \centering
    \includegraphics[width=0.7\linewidth]{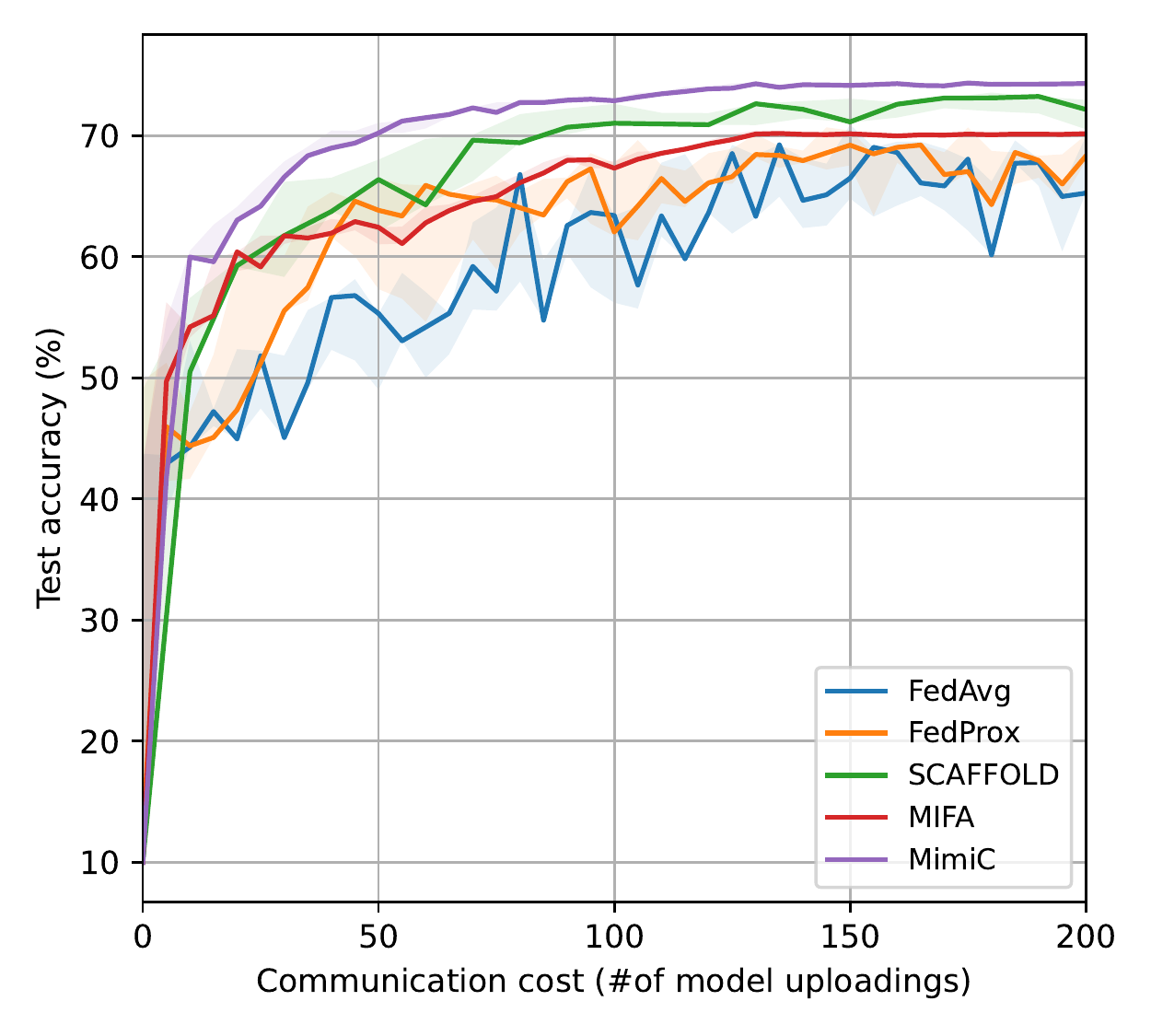}
    \caption{Test accuracy on the FMNIST dataset, where each client becomes active with a fixed probability of 0.1.}
    \label{fig:ind-f}
\end{figure}

\begin{figure}[t]
    \centering
    \includegraphics[width=0.7\linewidth]{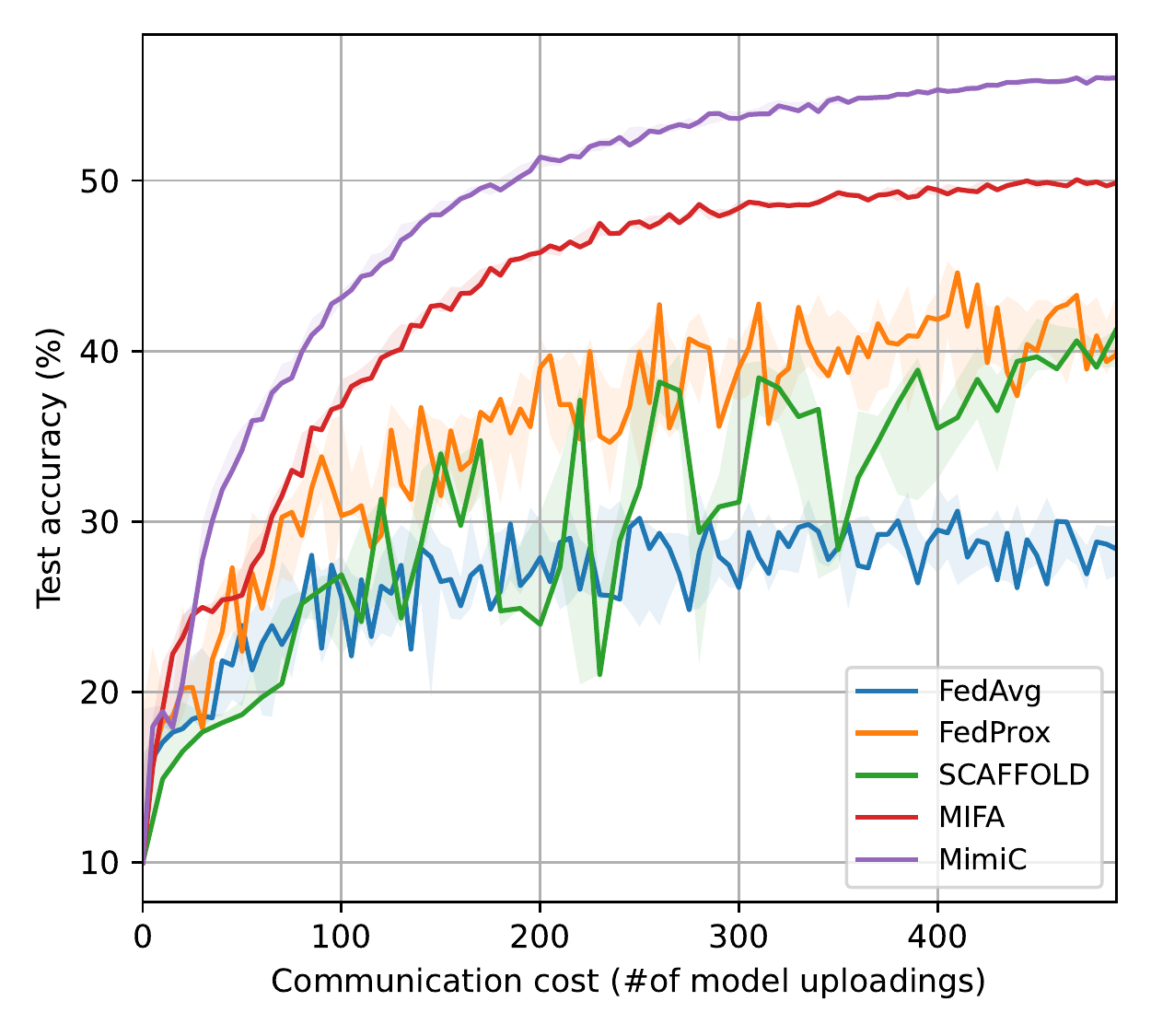}
    \caption{Test accuracy on the CIFAR-10 dataset, where each client becomes active with a fixed probability of 0.3.}
    \label{fig:ind-c}
\end{figure}

In addition, we simulate the scenario where clients have time-varying dropout probabilities, which are randomly generated in each iteration with parameter $P$ denoting the ratio of active clients in each iteration.
Table \ref{table:result} shows the test accuracy on the FMNIST and CIFAR-10 datasets. 
In general, all methods achieve higher test accuracy as the value of $P$ increases, i.e., more clients are active in each iteration.
Similar to the previous scenarios, MimiC consistently outperforms other algorithms with different levels of client availability. 
In particular, when the participation ratio is very low, i.e., $P=10\%$, MimiC achieves noticeable accuracy improvement compared to all baselines.
These results clearly demonstrate the effectiveness of MimiC in combating client dropouts in cross-device FL systems.

\begin{table*}[t]
\caption{Test accuracy after $200$ and $400$ model uploadings on the FMNIST and CIFAR-10 datasets, respectively, where $P\times N$ clients are active in each iteration.}
\label{table:result}
\centering
\begin{tabular}{|c|ccc|ccc|}
\hline
Dataset       & \multicolumn{3}{c|}{FMNIST}                                                                                                 & \multicolumn{3}{c|}{CIFAR-10}                                                                                      \\ \hline
P        & \multicolumn{1}{c|}{$10\%$}                    & \multicolumn{1}{c|}{$30\%$}                    & $50\%$                    & \multicolumn{1}{c|}{$10\%$}                    & \multicolumn{1}{c|}{$30\%$}                    & $50\%$           \\ \hline
FedAvg        & \multicolumn{1}{c|}{$64.71 \pm 2.26$}          & \multicolumn{1}{c|}{$69.22 \pm 6.19$}          & $71.31  \pm 1.09$         & \multicolumn{1}{c|}{$42.03 \pm 8.48$}          & \multicolumn{1}{c|}{$51.63 \pm 6.68$}          & $52.21 \pm 6.49$ \\ \hline
FedProx       & \multicolumn{1}{c|}{$66.91 \pm 7.24$}          & \multicolumn{1}{c|}{$69.43 \pm 5.05$}          & $72.91 \pm 1.00$          & \multicolumn{1}{c|}{$48.12 \pm 10.62$}         & \multicolumn{1}{c|}{$50.08 \pm 1.99$}          & $59.23 \pm 1.56$ \\ \hline
SCAFFOLD      & \multicolumn{1}{c|}{$63.77 \pm 4.66$}          & \multicolumn{1}{c|}{$73.18 \pm 3.74$}          & $71.88 \pm 0.25$          & \multicolumn{1}{c|}{$38.88 \pm 4.21$}          & \multicolumn{1}{c|}{$56.21 \pm 9.22$}          & $53.59 \pm 0.63$ \\ \hline
MIFA          & \multicolumn{1}{c|}{$68.58 \pm 6.45$}          & \multicolumn{1}{c|}{$73.43 \pm 5.61$}          & $73.24 \pm 0.41$          & \multicolumn{1}{c|}{$52.69 \pm 2.15$}          &  \multicolumn{1}{c|}{$55.62 \pm 6.10$}          & $56.14 \pm 3.33$ \\ \hline
MimiC         & \multicolumn{1}{c|}{$\mathbf{72.22 \pm 5.28}$} & \multicolumn{1}{c|}{$\mathbf{74.86 \pm 4.44}$} & $\mathbf{74.24 \pm 1.31}$ & \multicolumn{1}{c|}{$\mathbf{56.21 \pm 4.70}$} & \multicolumn{1}{c|}{$\mathbf{60.07 \pm 2.87}$} & $\mathbf{ 60.59 \pm 1.52}$ \\ \hline
\end{tabular}
\end{table*}

Finally, we study the impact of data heterogeneity on the training performance in Table \ref{table:class}, where we vary the number of classes of data samples on each client to simulate different degrees of data heterogeneity.
We see that more classes of data samples at each client, i.e., a lower level of data heterogeneity, can lead to a faster training speed due to a smaller objective mismatch.
When each client has $8$ or $10$ classes of data samples, different algorithms achieve similar model accuracy. This is because clients have almost IID data and the missing model updates of dropout clients can be easily compensated by those of the active clients.
In all cases, MimiC learns better models than the baselines methods as it alleviates the objective mismatch due to clients dropouts by appropriately correcting the local updates.

\begin{table}[ht]
\centering
\caption{Test accuracy after 400 model uploadings on the CIFAR-10 dataset, where each client becomes active with a fixed probability of 0.1.}
\label{table:class}
\begin{tabular}{|c|c|c|c|c|}
\hline
\# of classes/client & $2$ & $5$ & $8$ & $10$ \\ \hline
FedAvg              & $30.52$   & $77.76$    & $81.59$  & $82.10$   \\ \hline
FedProx             & $36.66$   & $77.74$    & $81.22$  & $81.96$     \\ \hline
SCAFFOLD            & $37.75$   & $78.15$    & $82.00$  & $82.74$    \\ \hline
MIFA                & $50.37$   & $77.0$6    & $81.84$  & $82.04$    \\ \hline
MimiC               & $\mathbf{52.91}$   & $\mathbf{78.44}$   & $\mathbf{82.16}$ & $\mathbf{82.93}$    \\ \hline
\end{tabular}
\end{table}

\section{Conclusions}\label{sec:conclusion}
In this paper, we investigated the cross-device federated learning in the presence of arbitrary client dropouts. We first proved that in this critical scenario, the classical FedAvg algorithm is not guaranteed to converge with the common choice of a decaying learning rate.
This unsatisfactory performance motivated the design of a new federated learning algorithm named MimiC, where each applied update is corrected via previous updates to mimic the imaginary central update without needing the client availability pattern.
We showed that with a proper choice of the learning rates, MimiC converges to a stationary point of the global loss function.
The extensive simulations demonstrated that MimiC consistently outperforms the baselines in various scenarios, in addition to having the theoretical convergence guarantee.
For future work, it is worth investigating how to fuse the staleness-aware model aggregation schemes with MimiC to further improve its training performance.

\bibliographystyle{IEEEtran}
\bibliography{ref}

\begin{IEEEbiography}[{\includegraphics[width=1in,height=1.25in,clip,keepaspectratio]{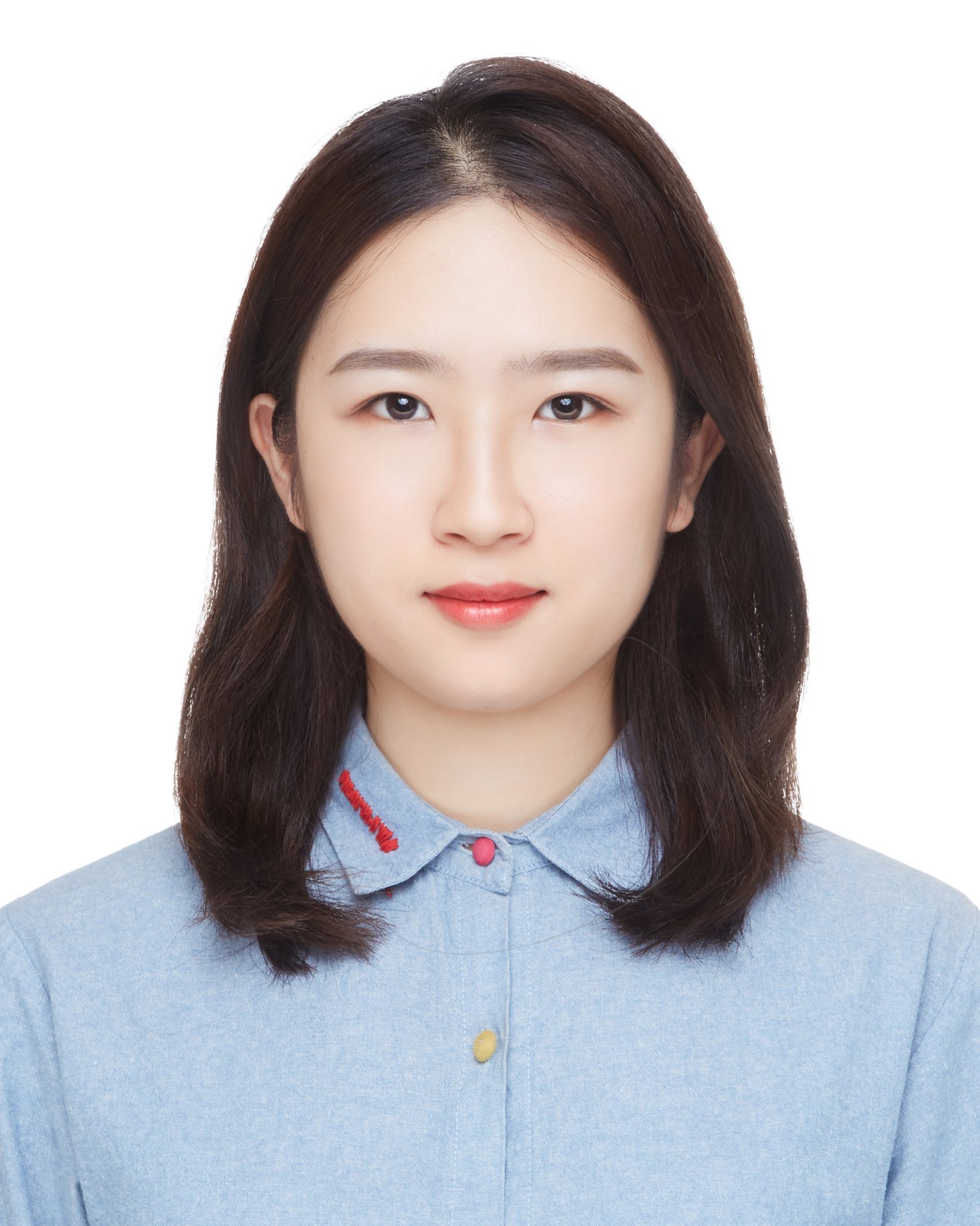}}]{Yuchang Sun}
(Graduate student member, IEEE) received the B.Eng. degree in electronic and information engineering from Beijing Institute of Technology in 2020. She is currently pursuing a Ph.D. degree at the Hong Kong University of Science and Technology. Her research interests include federated learning and distributed optimization.
\end{IEEEbiography}

\begin{IEEEbiography}[{\includegraphics[width=1in,height=1.25in,clip,keepaspectratio]{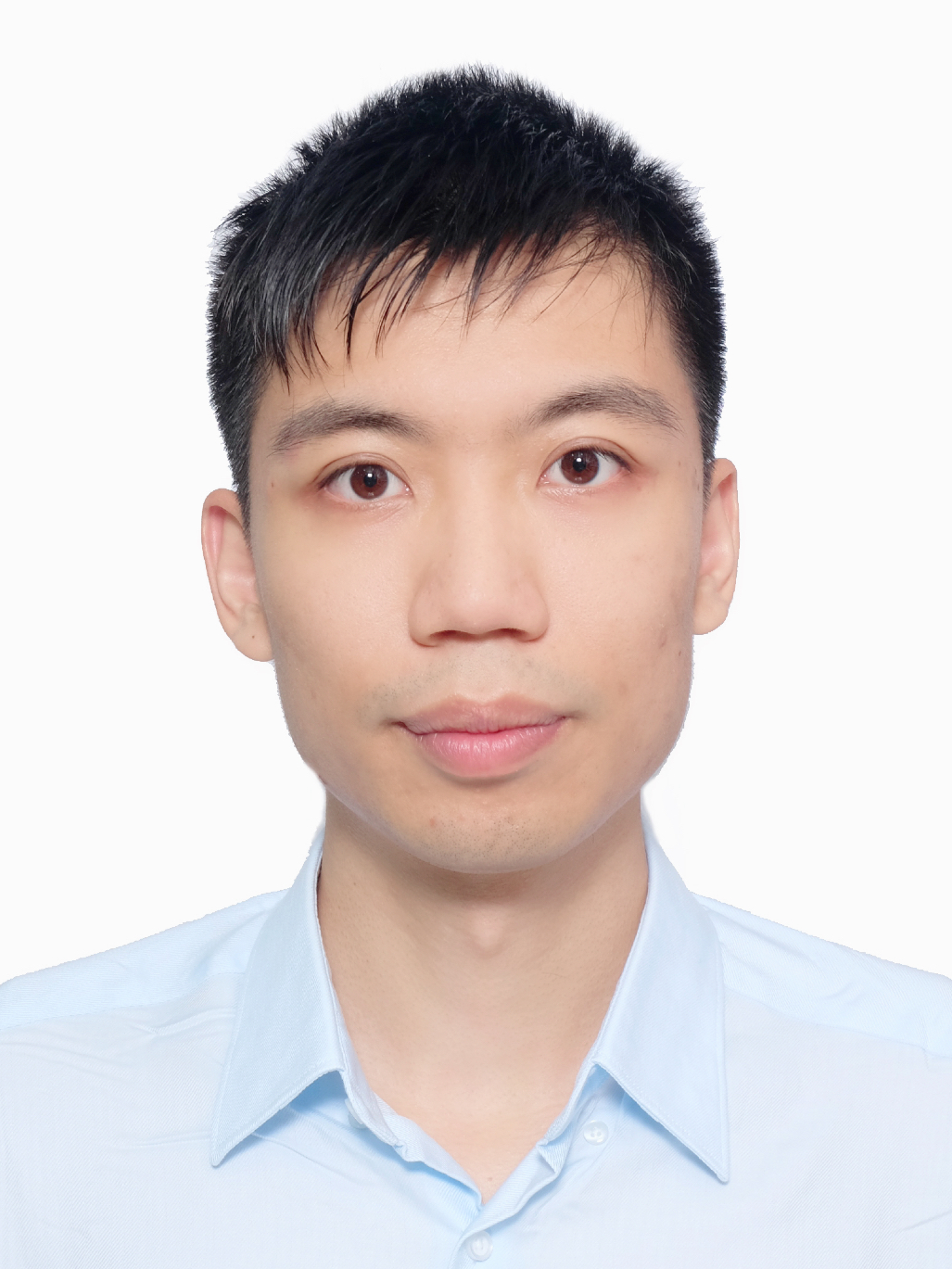}}]{Yuyi Mao}
(Member, IEEE) received the B.Eng. degree in information and communication engineering from Zhejiang University, Hangzhou, China, in 2013, and the Ph.D. degree in electronic and computer engineering from The Hong Kong University of Science and Technology, Hong Kong, in 2017. He was a Lead Engineer with the Hong Kong Applied Science and Technology Research Institute Co., Ltd., Hong Kong, and a Senior Researcher with the Theory Lab, 2012 Labs, Huawei Tech. Investment Co., Ltd., Hong Kong. He is currently a Research Assistant Professor with the Department of Electrical and Electronic Engineering, The Hong Kong Polytechnic University, Hong Kong. His research interests include wireless communications and networking, mobile-edge computing and learning, and wireless artificial intelligence.
He was the recipient of the 2021 IEEE Communications Society Best Survey Paper Award and the 2019 IEEE Communications Society and Information Theory Society Joint Paper Award. He was also recognized as an Exemplary Reviewer of the IEEE Wireless Communications Letters in 2021 and 2019 and the IEEE Transactions on Communications in 2020. He is an Editor of the IEEE Wireless Communications Letters and an Associate Editor of the EURASIP Journal on Wireless Communications and Networking.
\end{IEEEbiography}

\begin{IEEEbiography}[{\includegraphics[width=1in,height=1.25in,clip,keepaspectratio]{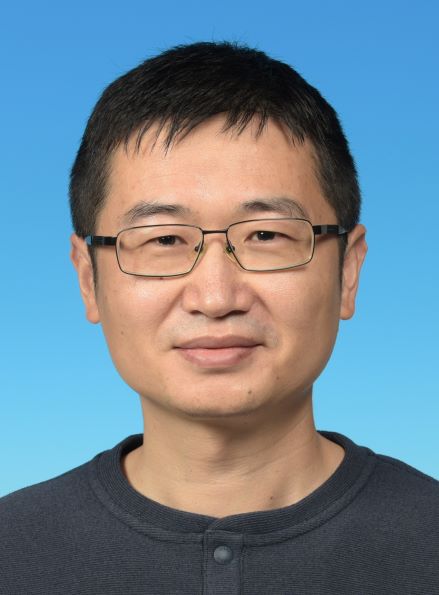}}]{Jun Zhang}
(Fellow, IEEE) received the B.Eng. degree in Electronic Engineering from the University of Science and Technology of China in 2004, the M.Phil. degree in Information Engineering from the Chinese University of Hong Kong in 2006, and the Ph.D. degree in Electrical and Computer Engineering from the University of Texas at Austin in 2009. He is an Associate Professor in the Department of Electronic and Computer Engineering at the Hong Kong University of Science and Technology. His research interests include wireless communications and networking, mobile edge computing and edge AI, and cooperative AI.

Dr. Zhang co-authored the book Fundamentals of LTE (Prentice-Hall, 2010). He is a co-recipient of several best paper awards, including the 2021 Best Survey Paper Award of the IEEE Communications Society, the 2019 IEEE Communications Society \& Information Theory Society Joint Paper Award, and the 2016 Marconi Prize Paper Award in Wireless Communications. Two papers he co-authored received the Young Author Best Paper Award of the IEEE Signal Processing Society in 2016 and 2018, respectively. He also received the 2016 IEEE ComSoc Asia-Pacific Best Young Researcher Award. He is an Editor of IEEE Transactions on Communications, IEEE Transactions on Machine Learning in Communications and Networking, and was an editor of IEEE Transactions on Wireless Communications (2015-2020). He served as a MAC track co-chair for IEEE Wireless Communications and Networking Conference (WCNC) 2011 and a co-chair for the Wireless Communications Symposium of IEEE International Conference on Communications (ICC) 2021. He is an IEEE Fellow and an IEEE ComSoc Distinguished Lecturer.
\end{IEEEbiography}

\clearpage
\appendices
\section{Proof of Lemma \ref{lem:E_t_1}}\label{proof:lem:E_t_1}
We begin with the definition of $\mathcal{E}_t^{\mathrm{F}}$ in FedAvg as follows:
\begin{align}
    \mathcal{E}_t^{\mathrm{F}}
    \triangleq & \mathbb{E} \left[ \left\| \mathbf{v}_t - \nabla f(\mathbf{w}_{t}) \right\|^2 \right] \nonumber \\
    = & \mathbb{E} \left[ \left\| \frac{1}{|\mathcal{S}_t|} \sum_{i\in\mathcal{S}_t} \frac{1}{K} \sum_{k=0}^{K-1} \nabla f_i(\mathbf{w}_{t,k}^i) - \nabla f(\mathbf{w}_{t}) \right\|^2 \right] \nonumber \\
    = & \mathbb{E} \left[ \left\| \frac{1}{|\mathcal{S}_t|} \sum_{i\in\mathcal{S}_t} \frac{1}{K} \sum_{k=0}^{K-1} \nabla f_i(\mathbf{w}_{t,k}^i) \right.\right. \nonumber \\
    & - \frac{1}{|\mathcal{S}_t|} \sum_{i\in\mathcal{S}_t} \frac{1}{K} \sum_{k=0}^{K-1} \nabla f_i(\mathbf{w}_{t}) \nonumber \\
    & \left.\left. + \frac{1}{|\mathcal{S}_t|} \sum_{i\in\mathcal{S}_t} \frac{1}{K} \sum_{k=0}^{K-1} \nabla f_i(\mathbf{w}_{t}) - \nabla f(\mathbf{w}_{t}) \right\|^2 \right] \nonumber \\
    \overset{(\text{a})}{\leq} & 2 \mathbb{E} \left[ \left\| \frac{1}{|\mathcal{S}_t|} \sum_{i\in\mathcal{S}_t} \frac{1}{K} \sum_{k=0}^{K-1} \left( \nabla f_i(\mathbf{w}_{t,k}^i) - \nabla f_i(\mathbf{w}_{t}) \right)  \right\|^2 \right] \nonumber \\
    & + 2 \underbrace{\mathbb{E} \left[ \left\| \frac{1}{|\mathcal{S}_t|} \sum_{i\in\mathcal{S}_t} \nabla f_i(\mathbf{w}_{t}) - \nabla f(\mathbf{w}_{t}) \right\|^2 \right]}_{\gamma_t} \nonumber \\
    \overset{(\text{b})}{\leq} & \frac{2}{|\mathcal{S}_t|} \sum_{i\in\mathcal{S}_t} \frac{1}{K} \sum_{k=0}^{K-1} \mathbb{E} \left[ \left\| \nabla f_i(\mathbf{w}_{t,k}^i) - \nabla f_i(\mathbf{w}_{t}) \right\|^2 \right] + 2 \gamma_t \nonumber \\
    \overset{(\text{c})}{\leq} & \frac{2L^2}{|\mathcal{S}_t|} \sum_{i\in\mathcal{S}_t} \frac{1}{K} \sum_{k=0}^{K-1} \mathbb{E} \left[ \left\| \mathbf{w}_{t,k}^i - \mathbf{w}_{t} \right\|^2 \right] + 2 \gamma_t \nonumber \\
    \overset{(\text{d})}{\leq} & \frac{2L^2}{|\mathcal{S}_t|} \sum_{i\in\mathcal{S}_t} \left[ 8\eta_L^2 K(K-1) \mathbb{E}[\| \nabla f(\mathbf{w}_{t}) \|^2]\right. \nonumber \\
    & + \left. 8\eta_L^2 K(K-1) \kappa_i^2 + 2\eta_L^2 (K-1) \sigma^2 \right] + 2 \gamma_t,
\end{align}
where both (a) and (b) follow the Cauchy–Schwarz inequality $\|\sum_{i=1}^n \mathbf{x}_i\|_2^2 \leq n \sum_{i=1}^n \| \mathbf{x}_i\|_2^2$, and (c) is due to the $L$-smoothness of the loss function.  Besides, (d) is derived using the following inequality, i.e., for any $i\in \mathcal{S}_t, 0\leq k \leq K-1$, we have:
\begin{align}
    \mathbb{E} \left[ \left\| \mathbf{w}_{t,k}^i - \mathbf{w}_{t} \right\|^2 \right]
    \leq & \frac{8 \eta_L^{2}(K-1)}{K} \mathbb{E}\left[\left\|\nabla f\left(\mathbf{w}_{t}\right)\right\|^{2}\right] \nonumber \\
    + & \frac{8(K-1) \eta_L^{2} \kappa_{i}^2}{K}+\frac{2(K-1) \eta_L^{2} \sigma^{2}}{K^{2}},
\end{align}
which can be obtained using a similar proof of Lemma C.2 in \cite{mifa}.
The main difference falls on different assumptions of data heterogeneity, i.e., Assumption \ref{ass-3} in this study and Assumption 7 in \cite{mifa}.
\qed

\section{Proof of Theorem \ref{convergence-fedavg}}\
\label{proof:thm:fedavg}
We rearrange the terms in \eqref{eq:one-round} to obtain:
\begin{align}
    & \frac{\eta_t}{2} \mathbb{E} \left[\left\| \nabla f(\mathbf{w}_t) \right\|^2 \right] 
    \label{eq:help-31} \\
    \leq &  \mathbb{E} [f(\mathbf{w}_{t})] - \mathbb{E} [f(\mathbf{w}_{t+1})]  - \left(\frac{\eta_t}{2} - \frac{\eta_t^2 L}{2}\right) \left\| \mathbb{E}[\mathbf{v}_t] \right\|^2  \nonumber \\
    + & \frac{\eta_t^2 L}{2} \Phi_t^{\mathrm{F}}
    + \frac{\eta_t}{2} \mathcal{E}_t^{\mathrm{F}} \nonumber \\
    \leq &  \mathbb{E} [f(\mathbf{w}_{t})] - \mathbb{E} [f(\mathbf{w}_{t+1})] + \frac{\eta_t^2 L}{2} \Phi_t^{\mathrm{F}}
    + \frac{\eta_t}{2} \mathcal{E}_t^{\mathrm{F}}. \nonumber
\end{align}
Now we substitute the terms $\mathcal{E}_t^{\mathrm{F}}$ and $\Phi_t$ by their upper bounds in \eqref{eq:E_t-fedavg} and \eqref{eq:fedavg:phi}, respectively, and sum up both sides of \eqref{eq:help-31} over $t=0,1,\dots,T-1$ as follows:
\begin{align}
    & \sum_{t=0}^{T-1} \frac{\eta_t}{2} \mathbb{E} \left[\left\| \nabla f(\mathbf{w}_t) \right\|^2 \right] 
    \label{eq:help-32} \\
    \leq & \mathbb{E} [f(\mathbf{w}_{0})] - \mathbb{E} [f(\mathbf{w}_{T})]
    + \frac{\eta_t^2 L}{2} \frac{\sigma^2}{|\mathcal{S}_t|K} \nonumber \\
    +& \frac{\eta_t}{2} \frac{2L^2}{|\mathcal{S}_t|} \left[ \sum_{i\in\mathcal{S}_t} \left( 8\eta_L^2 K(K-1) \mathbb{E}[\| \nabla f(\mathbf{w}_{t}) \|^2] \right. \right. \nonumber \\
    +& \left. \left. 8\eta_L^2 K(K-1) \kappa_i^2 + 2\eta_L^2 (K-1) \sigma^2 \right)
    + 2 \gamma_t \right]. \nonumber
\end{align}
By moving the term $\mathbb{E}[\| \nabla f(\mathbf{w}_{t}) \|^2]$ to the left-hand side (LHS) of \eqref{eq:help-32}, we have:
\begin{align}
    & \sum_{t=0}^{T-1} \frac{\eta_t}{2} \left(1 - 16\eta_L^2 L^2 K(K-1) \right) \mathbb{E} \left[\left\| \nabla f(\mathbf{w}_t) \right\|^2 \right] 
    \label{eq:help-33} \\
    \leq & \mathbb{E} [f(\mathbf{w}_{0})] - \mathbb{E} [f(\mathbf{w}_{T})]
    + \frac{\eta_t^2 L}{2} \frac{\sigma^2}{|\mathcal{S}_t|K} \nonumber \\
    +& \frac{\eta_t L^2}{|\mathcal{S}_t|} \left[ \sum_{i\in\mathcal{S}_t} \left( 8\eta_L^2 K(K-1) \kappa_i^2 + 2\eta_L^2 (K-1) \sigma^2 \right)
    + 2 \gamma_t \right] \nonumber \\
    \leq & \mathbb{E} [f(\mathbf{w}_{0})] - \mathbb{E} [f(\mathbf{w}_{*})]
    + \left( \frac{\eta_t^2 L}{2|\mathcal{S}_t|K} + 2 \eta_t\eta_L^2 L^2 (K-1) \right) \sigma^2 \nonumber \\
    +& \frac{\eta_t L^2}{|\mathcal{S}_t|}  8\eta_L^2 K(K-1) \sum_{i\in\mathcal{S}_t} \kappa_i^2
    + \frac{2\eta_t L^2}{|\mathcal{S}_t|} \gamma_t. \nonumber
\end{align}
Let $C_K \triangleq 16 \eta_L^2 L^2 K(K-1)$. The proof is completed by dividing both sides of \eqref{eq:help-33} by $\sum_{t=0}^{T-1} \frac{\eta_t}{2}$.

\qed

\section{Proof of Lemma \ref{lem:E_t_2}}\label{proof:lem:E_t_2}
We first derive an upper bound of $\mathcal{E}_t^{\mathrm{M}}$ as follows:
\begin{align}
    \mathcal{E}_t^{\mathrm{M}}
    \overset{(\text{a})}{\leq} & \frac{1}{|\mathcal{S}_t|} \sum_{i\in\mathcal{S}_t} \mathbb{E}\left[\left\| \mathbb{E}[\hat{g}_i(\mathbf{w}_t)] - \mathbb{E}[\hat{g}_i(\mathbf{w}_{t-\tau(t,i)})] \right.\right. \nonumber \\
    & + \mathbb{E}[\mathbf{v}_{t-\tau(t,i)}] - \nabla f(\mathbf{w}_{t-\tau(t,i)}) \nonumber \\
    & \left.\left. + \nabla f(\mathbf{w}_{t-\tau(t,i)}) - \nabla f(\mathbf{w}_t)  \right\|^2\right] \nonumber \\
    \overset{(\text{b})}{\leq} & \frac{1}{|\mathcal{S}_t|} \sum_{i\in\mathcal{S}_t} \left(1+\alpha_t \right) \mathbb{E}\left[\left\| \mathbb{E}[\mathbf{v}_{t-\tau(t,i)}] - \nabla f(\mathbf{w}_{t-\tau(t,i)}) \right\|^2\right] \nonumber \\
    & + \left(1+\frac{1}{\alpha_t} \right) \mathbb{E}\left[\left\| \mathbb{E}[\hat{g}_i(\mathbf{w}_t)] - \mathbb{E}[\hat{g}_i(\mathbf{w}_{t-\tau(t,i)})] \right.\right. \nonumber \\
    & \left.\left. + \nabla f(\mathbf{w}_{t-\tau(t,i)}) - \nabla f(\mathbf{w}_t) \right\|^2\right] \nonumber \\
    \overset{(\text{c})}{\leq} & \frac{1}{|\mathcal{S}_t|} \sum_{i\in\mathcal{S}_t} (1+\alpha_t) \underbrace{ \mathbb{E}\left[\left\| \mathbb{E}[\mathbf{v}_{t-\tau(t,i)}] - \nabla f(\mathbf{w}_{t-\tau(t,i)}) \right\|^2\right]}_{\mathcal{E}_{t,1}^{i}} \nonumber \\
    & + 2 \left(1+\frac{1}{\alpha_t} \right) \underbrace{\mathbb{E}\left[\left\| \mathbb{E}[\hat{g}_i(\mathbf{w}_t)] - \mathbb{E}[\hat{g}_i(\mathbf{w}_{t-\tau(t,i)})] \right\|^2\right]}_{\mathcal{E}_{t,2}^{i}} \nonumber \\
    & + 2\left(1+\frac{1}{\alpha_t} \right) \underbrace{ \mathbb{E}\left[\left\| \nabla f(\mathbf{w}_{t-\tau(t,i)}) - \nabla f(\mathbf{w}_t) \right\|^2\right]}_{\mathcal{E}_{t,3}^{i}}, \label{eq:E}
\end{align}
where (a) follows the Cauchy–Schwarz inequality, while (b) and (c) follow the fact that $\|\mathbf{x}+\mathbf{y} \|^2 \leq (1+\alpha) \|\mathbf{x}\|^2 + (1+\frac{1}{\alpha}) \|\mathbf{y}\|^2, \forall \alpha >0$ with $\alpha=\alpha_t$ and $\alpha=1$, respectively.
To make $\alpha_t > 0$, the condition in \eqref{eq:lr-1}, i.e., $\rho_t \triangleq \frac{\eta_t|\mathcal{S}_{t+1}|}{\eta_{t+1}|\mathcal{S}_t|} >1$, is required.
We notice that 
\begin{equation}
    \mathcal{E}_{t,1}^{i} = \mathbb{E}\left[\left\| \mathbf{v}_{t-\tau(t,i)} - \nabla f(\mathbf{w}_{t-\tau(t,i)}) \right\|^2\right] = \mathcal{E}_{t-\tau(t,i)},
\label{eq:E-1}
\end{equation}
and respectively provide upper bounds for the terms $\mathcal{E}_{t,2}^{i}$ and $\mathcal{E}_{t,3}^{i}$ in the following. For the term $\mathcal{E}_{t,2}^{i}$, we have
\begin{align}
    & \mathcal{E}_{t,2}^{i} \nonumber \\
    = & \mathbb{E}\left[ \left\| \mathbb{E}[\hat{g}_i(\mathbf{w}_t)] - \mathbb{E}[\hat{g}_i(\mathbf{w}_{t-\tau(t,i)})] \right\|^2 \right] \nonumber \\
    = & \mathbb{E}\left[\left\| \frac{1}{K} \sum_{k=0}^{K-1} \nabla f_i(\mathbf{w}_{t,k}^{i}) - \frac{1}{K} \sum_{k=0}^{K-1} \nabla f_i(\mathbf{w}_{t-\tau(t,i),k}^{i}) \right\|^2 \right] \nonumber \\
    \overset{(\text{d})}{\leq} & L^2 \mathbb{E}\left[\left\| \frac{1}{K} \sum_{k=0}^{K-1} \mathbf{w}_{t,k}^{i} - \mathbf{w}_{t-\tau(t,i),k}^{i} \right\|^2 \right] \nonumber \\
    \overset{(\text{e})}{\leq} & \frac{L^2}{K} \sum_{k=0}^{K-1} \mathbb{E}\left[\left\| \mathbf{w}_{t,k}^{i} - \mathbf{w}_{t-\tau(t,i),k}^{i} \right\|^2 \right]\nonumber \\
    = & \frac{ L^2}{K} \sum_{k=0}^{K-1} \mathbb{E}\left[\left\| \left(\mathbf{w}_{t} - \eta_L \sum_{s=0}^{k-1} \nabla F_i(\mathbf{w}_{t,s}^{i};\mathbf{\xi}_{t,s}^{i}) \right) \right.\right. \nonumber \\
    & \left.\left. - \left( \mathbf{w}_{t-\tau(t,i)} - \eta_L \sum_{s=0}^{k-1} \nabla F_i(\mathbf{w}_{t-\tau(t,i),s}^{i};\mathbf{\xi}_{t-\tau(t,i),s}^{i}) \right) \right\|^2\right] \nonumber \\
    \overset{(\text{f})}{\leq} & \frac{L^2}{K} \sum_{k=0}^{K-1} (2+2\eta_L^2L^2)^k \underbrace{\mathbb{E}\left[\left\| \mathbf{w}_{t} - \mathbf{w}_{t-\tau(t,i)} \right\|^2\right]}_{D_t^i} \nonumber \\
    = & \frac{L^2[(2+2\eta_L^2L^2)^K - 1]}{K(2\eta_L^2L^2 + 1)} D_t^i, \label{eq:E-2}
\end{align}
where (d) is due to the $L$-smoothness in Assumption \ref{ass-1}, and (e) holds because of the Cauchy–Schwarz inequality.
Also, (f) applies the following result:
\begin{align}
    D_t^i 
    &\overset{(\text{g})}{=} \mathbb{E}\left[\left\| \sum_{s=1}^{\tau(t,i)} \eta_{t-s} \mathbf{v}_{t-s} \right\|^2\right] \nonumber \\
    &\overset{(\text{h})}{\leq} \tau(t,i) \sum_{s=1}^{\tau(t,i)}\eta_{t-s}^2 \mathbb{E}\left[\left\| \mathbf{v}_{t-s} \right\|^2\right] \nonumber \\
    &\overset{(\text{i})}{=} \tau(t,i) \sum_{s=1}^{\tau(t,i)}\eta_{t-s}^2 \left( \frac{\sigma^2}{|\mathcal{S}_{t-s}|K} + \left\| \mathbb{E}[\mathbf{v}_{t-s}] \right\|^2 \right),
\end{align}
where (g) is by the evolutionary of $\mathbf{w}_{t} \leftarrow  \mathbf{w}_{t-\tau(t,i)} - \sum_{s=1}^{\tau(t,i)} \eta_{t-s} \mathbf{v}_{t-s}$, (h) follows the Cauchy-Schwarz Inequality $\|\sum_{i=1}^n \mathbf{x}_i\|_2^2 \leq n \sum_{i=1}^n \| \mathbf{x}_i\|_2^2$, and (i) applies Lemma 2 in \cite{wang2020tackling}.

For the term $\mathcal{E}_{t,3}^{i}$, by using the $L$-smoothness in Assumption \ref{ass-1}, we have:
\begin{align}
    \mathcal{E}_{t,3}^{i} &= \mathbb{E}\left[\left\| \nabla f(\mathbf{w}_{t-\tau(t,i)}) - \nabla f(\mathbf{w}_t) \right\|^2\right] \nonumber \\
    &\leq L^2 \mathbb{E}\left[\left\| \mathbf{w}_{t-\tau(t,i)} - \mathbf{w}_t \right\|^2\right] \nonumber \\
    &= L^2 D_t^i. \label{eq:E-3}
\end{align}
By upper bounding the RHS of \eqref{eq:E} using \eqref{eq:E-1}, \eqref{eq:E-2}, and \eqref{eq:E-3}, we arrive at the desired result.
\qed

\section{Proof of Theorem \ref{thm:mimic}}\label{proof:thm:mimic}
Recall the one-round progress in \eqref{eq:one-round}, the following inequality holds:
\begin{align}
    &\mathbb{E} [f(\mathbf{w}_{t+1})] + \left(\frac{\eta_t}{2} - \frac{\eta_t^2 L}{2} \right) \left\| \mathbb{E}[\mathbf{v}_t] \right\|^2 \nonumber \\
    \overset{(\text{a})}{\leq} &\mathbb{E} [f(\mathbf{w}_{t})] - \frac{\eta_t}{2} \mathbb{E} \left[ \left\| \nabla f(\mathbf{w}_t) \right\|^2 \right]+ \frac{\eta_t^2 L}{2|\mathcal{S}_t|K} \sigma^2 + \frac{\eta_t}{2} \mathcal{E}_t^{\mathrm{M}}, \label{eq:one-round-mimic}
\end{align}
where (a) applies Lemma \ref{lem:phi_t_2}.

By adding $\frac{\rho_t \eta_{t}}{2\nu} \mathcal{E}_t^{\mathrm{M}}$ at the both sides of \eqref{eq:one-round-mimic} and applying the upper bound of $\mathcal{E}_t^{\mathrm{M}}$ in Lemma \ref{lem:E_t_2}, we obtain:
\begin{align}
    &\mathbb{E} [f(\mathbf{w}_{t+1})] + \left(\frac{\eta_t}{2} - \frac{\eta_t^2 L}{2} \right) \left\| \mathbb{E}[\mathbf{v}_t] \right\|^2 + \frac{\rho_t\eta_{t}}{2\nu} \mathcal{E}_t^{\mathrm{M}} \nonumber \\
    \leq & \mathbb{E} [f(\mathbf{w}_{t})] - \frac{\eta_t}{2} \mathbb{E} \left[ \left\| \nabla f(\mathbf{w}_t) \right\|^2 \right] \nonumber \\
    + & \frac{\eta_t^2 L}{2|\mathcal{S}_t|K} \sigma^2 + \frac{\eta_{t}}{2} \mathcal{E}_t^{\mathrm{M}} + \frac{\rho_t (\rho_t-\nu)}{2\nu} \frac{\eta_t}{|\mathcal{S}_t|} \sum_{i\in\mathcal{S}_t} \mathcal{E}_{t-\tau(t,i)} \nonumber \\
    + & \frac{\rho_t(\rho_t-\nu)}{\nu(\rho_t-1-\nu)} \frac{\eta_t \varphi_K}{|\mathcal{S}_t|K} \sum_{i\in\mathcal{S}_t}
    \tau(t,i) \sum_{s=1}^{\tau(t,i)} \frac{\eta_{t-s}^2 \sigma^2}{|\mathcal{S}_{t-s}|} \nonumber \\
    + & \frac{\rho_t(\rho_t-\nu)}{\nu(\rho_t-1-\nu)} \frac{\eta_t \varphi_K}{|\mathcal{S}_t|} \sum_{i\in\mathcal{S}_t}
    \tau(t,i) \sum_{s=1}^{\tau(t,i)}\eta_{t-s}^2 \left\| \mathbb{E}[\mathbf{v}_{t-s}] \right\|^2. \label{eq:one-round-mimic-2}
\end{align}

\noindent
By summing up both sides of \eqref{eq:one-round-mimic-2} over $t=0,1,\dots,T-1$, we have:
\begin{align}
    & \sum_{t=0}^{T-1} \mathbb{E} [f(\mathbf{w}_{t+1})] + \sum_{t=0}^{T-1} \left(\frac{\eta_t}{2} - \frac{\eta_t^2 L}{2} \right) \left\| \mathbb{E}[\mathbf{v}_t] \right\|^2 + \sum_{t=0}^{T-1}  \frac{\rho_t\eta_{t}}{2\nu} \mathcal{E}_t^{\mathrm{M}} \nonumber \\
    \leq & \sum_{t=0}^{T-1} \mathbb{E} [f(\mathbf{w}_{t})] - \sum_{t=0}^{T-1} \frac{\eta_t}{2} \mathbb{E} \left[ \left\| \nabla f(\mathbf{w}_t) \right\|^2 \right]+ \sum_{t=0}^{T-1} \frac{\eta_t^2 L}{2|\mathcal{S}_t|K} \sigma^2 \nonumber \\
    + & \sum_{t=0}^{T-1} \frac{\eta_{t}}{2} \mathcal{E}_t^{\mathrm{M}} + \sum_{t=0}^{T-1} \frac{\rho_t (\rho_t-\nu)}{2\nu} \frac{\eta_t}{|\mathcal{S}_t|} \sum_{i\in\mathcal{S}_t} \mathcal{E}_{t-\tau(t,i)} \nonumber \\
    + & \sum_{t=0}^{T-1} \frac{\rho_t(\rho_t-\nu)}{\nu(\rho_t-1-\nu)} \frac{\eta_t \varphi_K}{|\mathcal{S}_t|K} \sum_{i\in\mathcal{S}_t}
    \tau(t,i) \sum_{s=1}^{\tau(t,i)} \frac{\eta_{t-s}^2 \sigma^2}{|\mathcal{S}_{t-s}|} \nonumber \\
    + & \sum_{t=0}^{T-1} \frac{\rho_t(\rho_t-\nu)}{\nu(\rho_t-1-\nu)} \frac{\eta_t \varphi_K}{|\mathcal{S}_t|} \sum_{i\in\mathcal{S}_t}
    \tau(t,i) \sum_{s=1}^{\tau(t,i)}\eta_{t-s}^2 \left\| \mathbb{E}[\mathbf{v}_{t-s}] \right\|^2. \label{eq:help-0}
\end{align}

\noindent
Denote $\Tilde{\tau}(t) \triangleq \max_{i\in \mathcal{S}_t} \tau(t,i)$ as the maximum dropout iterations of active clients in iteration $t$.
We rearrange the terms in \eqref{eq:help-0} as follows:
\begin{align}
    & \sum_{t=0}^{T-1} \mathbb{E} [f(\mathbf{w}_{t+1})] + \sum_{t=0}^{T-1} \left(\frac{\eta_t}{2} - \frac{\eta_t^2 L}{2} \right) \left\| \mathbb{E}[\mathbf{v}_t] \right\|^2 + \sum_{t=0}^{T-1} \frac{\rho_t\eta_{t}}{2\nu} \mathcal{E}_t^{\mathrm{M}} \nonumber \\
    \leq & \sum_{t=0}^{T-1} \mathbb{E} [f(\mathbf{w}_{t})] - \sum_{t=0}^{T-1} \frac{\eta_t}{2} \mathbb{E} \left[ \left\| \nabla f(\mathbf{w}_t) \right\|^2 \right]+ \sum_{t=0}^{T-1} \frac{\eta_t^2 L}{2|\mathcal{S}_t|K} \sigma^2 \nonumber \\
    + & \sum_{t=0}^{T-1} \left[ \frac{\eta_{t}}{2} +  \frac{\rho_t (\rho_t-\nu)}{2\nu} \sum_{i\in\mathcal{S}_{t}} \sum_{s=1}^{\Tilde{\tau}(t)} \frac{\eta_{t+s}}{|\mathcal{S}_{t+s}|} \mathbf{1}\{ \tau(t+s,i)=s \} \right] \mathcal{E}_t^{\mathrm{M}} \nonumber \\
    + & \sum_{t=0}^{T-1} \frac{\rho_t(\rho_t-\nu)}{\nu(\rho_t-1-\nu)}  \frac{\eta_t \varphi_K}{|\mathcal{S}_t|K} \sum_{i\in\mathcal{S}_t}
    \tau(t,i) \sum_{s=1}^{\tau(t,i)} \frac{\eta_{t-s}^2 \sigma^2}{|\mathcal{S}_{t-s}|} \nonumber \\
    + & \sum_{t=0}^{T-1} \frac{\rho_t(\rho_t-\nu)}{\nu(\rho_t-1-\nu)} \frac{\eta_t \varphi_K}{|\mathcal{S}_t|} \sum_{i\in\mathcal{S}_t} \tau(t,i) \sum_{s=1}^{\tau(t,i)}\eta_{t-s}^2 \left\| \mathbb{E}[\mathbf{v}_{t-s}] \right\|^2. \label{eq:help-1}
\end{align}

We now compare the terms in the LHS and RHS of \eqref{eq:help-1}.
Firstly, to eliminate the terms related to $\{\mathcal{E}_t^{\mathrm{M}}\}$'s, we need to verify $\sum_{t=0}^{T-1}\mathcal{E}_t^{\mathrm{M}} \frac{\rho_t\eta_{t}}{2\nu} \geq \sum_{t=0}^{T-1}\mathcal{E}_t^{\mathrm{M}} ( \frac{\eta_{t}}{2} + \frac{\rho_t (\rho_t-\nu)}{2\nu} \sum_{i\in\mathcal{S}_{t}} \sum_{s=1}^{\tilde{\tau}(t)} \frac{\eta_{t+s}}{|\mathcal{S}_{t+s}|} \mathbf{1}\{ \tau(t+s,i)=s \} )$, or equivalently:
\begin{equation}
     \left(\frac{\rho_t}{2\nu} - \frac{1}{2}\right) \eta_{t} \geq \frac{\rho_t (\rho_t-\nu)}{2\nu} \sum_{i\in\mathcal{S}_{t}} \sum_{s=1}^{\tilde{\tau}(t)} \frac{\eta_{t+s}}{|\mathcal{S}_{t+s}|} \mathbf{1}\{ \tau(t+s,i)=s \}.
     \label{eq:R2}
\end{equation}
It is straightforward that $\max_{s\in\{1,\cdots,\tilde{\tau}(t)\}} \frac{\eta_{t+s}}{|\mathcal{S}_{t+s}|} = \frac{\eta_{t+1}}{|\mathcal{S}_{t+1}|}$.
Since $\tau(t+s,i)=s$ holds at most once from iteration $t$ to $t+\tilde{\tau}(t)$, the RHS of \eqref{eq:R2} is upper bounded as $ \frac{\rho_t (\rho_t-\nu)}{2\nu} |\mathcal{S}_{t}| \frac{\eta_{t+1}}{|\mathcal{S}_{t+1}|}$. Then, it remains to show $\frac{\rho_t - \nu}{2\nu} \frac{\eta_{t}}{|\mathcal{S}_{t}|} \geq \frac{\rho_t-\nu}{2\nu} \rho_t \frac{\eta_{t+1}}{|\mathcal{S}_{t+1}|}$. Note that the terms related to $\mathcal{E}_t^{\mathrm{M}}$ can be eliminated if $\rho_t - \nu>0$, namely the condition in \eqref{eq:lr-1}.

The terms related to $\eta_t^2 \left\| \mathbb{E}[\mathbf{v}_t] \right\|^2$ in both sides of \eqref{eq:help-1} can be cancelled out if the following condition holds:
\begin{align}
    & \frac{1}{2\eta_t} - \frac{L}{2} 
    \label{eq:R3} \\
    \geq & \frac{\rho_t(\rho_t-\nu)}{2\nu} \sum_{j=t+1}^{t+\Tilde{\tau}(t)} \frac{\eta_j \varphi_K}{|\mathcal{S}_j|} 
    \sum_{i\in\mathcal{S}_j} \tau(j,i) \sum_{s=1}^{\tau(j,i)} \mathbf{1}\{\tau(j+s,i)= s \}. \nonumber
\end{align}
In \eqref{eq:R3}, the RHS can be upper bounded as $\frac{\rho_t(\rho_t-\nu)}{2\nu} \frac{\eta_{t+1}}{|\mathcal{S}_{t+1}|} \varphi_K \tau_{\max} N = \frac{\rho_t-\nu}{2\nu} \frac{\eta_{t}}{|\mathcal{S}_{t}|} \varphi_K \tau_{\max} N$.
This condition can be satisfied with small enough learning rates such that $\frac{1}{\eta_{t}} \left( \frac{1}{2\eta_t} - \frac{L}{2} \right) \geq \frac{\rho_t-\nu}{2\nu} \frac{\varphi_K \tau_{\max} N}{|\mathcal{S}_{t}|}$.

By rearranging the terms of \eqref{eq:help-1}, we obtain:
\begin{align}
    & \sum_{t=0}^{T-1} \frac{\eta_t}{2} \mathbb{E} \left\| \nabla f(\mathbf{w}_t) \right\|^2 \nonumber \\
    \leq & \mathbb{E} [f(\mathbf{w}_{0})] - \mathbb{E} [f(\mathbf{w}_{T})] + \sum_{t=0}^{T-1} \frac{\sigma^2}{|\mathcal{S}_t|K} \nonumber\\
    & \times \left( \frac{\eta_t^2 L}{2} + \frac{\rho_t(\rho_t-\nu)}{\nu(\rho_t-1-\nu)} \eta_t \varphi_K \sum_{i\in\mathcal{S}_t} \tau(t,i) \sum_{s=1}^{\tau(t,i)}  \frac{\eta_{t-s}^2}{|\mathcal{S}_{t-s}|} \right) \nonumber\\
    \overset{(\text{b})}{\leq} & \mathbb{E} [f(\mathbf{w}_{0})] - \mathbb{E} [f(\mathbf{w}^*)] \nonumber\\
    & + \sum_{t=0}^{T-1} \frac{\sigma^2}{|\mathcal{S}_t|K} \left( \frac{\eta_t^2 L}{2} + \frac{\rho_t(\rho_t-\nu)}{\nu(\rho_t-1-\nu)} \eta_t \varphi_K \Tilde{\tau}^2(t) \frac{\eta_{t+1}^2}{|\mathcal{S}_{t+1}|}\right) \nonumber \\
    \leq & \mathbb{E} [f(\mathbf{w}_{0})] - \mathbb{E} [f(\mathbf{w}^*)] \nonumber\\
    & + \sum_{t=0}^{T-1} \frac{\sigma^2}{|\mathcal{S}_t|K} \left( \frac{\eta_t^2 L}{2} + \frac{\rho_t(\rho_t-\nu)}{\nu(\rho_t-1-\nu)} \eta_t \varphi_K \tau_{\text{max}}^2 \frac{\eta_{t+1}^2}{|\mathcal{S}_{t+1}|}\right). \label{eq:help-2}
\end{align}
Finally, dividing both sides of \eqref{eq:help-2} by $\sum_{t=0}^{T-1} \frac{\eta_t}{2}$ yields the result in \eqref{eq:thm-mimic}.
\qed

\section{Proof of Lemma \ref{lem:lr}}\label{proof:lr}

We construct the global learning rates as
\begin{equation}
    \eta_t = \frac{c|\mathcal{S}_{t}|}{t+\beta}, \forall t \in [T],
\end{equation}
where $c>0$ and $\beta>0$ are positive constants.
According to the definition of $\rho_t$ in \eqref{eq:lr-1}, one can easily verify that :
\begin{equation}
    \rho_t = \frac{t+\beta+1}{t+\beta} >1.
\end{equation}
It is also straightforward to show:
\begin{equation}
    \lim_{T\rightarrow \infty} \sum_{t=0}^{T-1} \eta_t = \lim_{T\rightarrow \infty} c \sum_{t=0}^{T-1} \frac{|\mathcal{S}_{t}|}{t+\beta} = \infty,
\end{equation}
and 
\begin{align}
    \lim_{T\rightarrow \infty} \sum_{t=0}^{T-1} \eta_t^2 
    = &\lim_{T\rightarrow \infty} c^2 \sum_{t=0}^{T-1} \frac{|\mathcal{S}_{t}|^2}{(t+\beta)^2} \nonumber \\
    \leq &\lim_{T\rightarrow \infty} c^2 N^2 \sum_{t=0}^{T-1} \frac{1}{(t+\beta)^2} \nonumber \\
    < &\infty.
\end{align}
Moreover, to verify the condition in \eqref{eq:lr-2}, i.e.,
\begin{equation}
    \frac{t+\beta}{c|\mathcal{S}_{t}|} \left( \frac{t+\beta}{2c|\mathcal{S}_{t}|} - \frac{L}{2} \right) \geq \frac{(\frac{t+\beta+1}{t+\beta}-\nu)\varphi_K \tau_{\max} N}{2\nu|\mathcal{S}_{t}|},
    \label{eq:help-13}
\end{equation}
we need to show:
\begin{equation}
    \frac{t+\beta}{c} \left( \frac{t+\beta}{cN} - L \right) \geq \frac{\varphi_K \tau_{\max} N}{\nu} \left( 1-\nu + \frac{1}{t+\beta} \right),
    \label{eq:help-14}
\end{equation}
since $\frac{1}{|\mathcal{S}_{t}|} \geq \frac{1}{N}$.
Next, we rewrite \eqref{eq:help-14} as the following quadratic inequality:
\begin{equation}
     A_t c^2 + B_t c - C_t \leq 0,
\end{equation}
where $A_t \triangleq \varphi_K \tau_{\max} N^2 \left[ (1-\nu)(t+\beta) +1 \right]$, $B_t \triangleq \nu LN(t+\beta)^2$, and $C_t \triangleq \nu(t+\beta)^3$.
Since $\sqrt{B_t^2 + 4A_tC_t} > B_t, \forall t \in [T]$, we can show the inequality in \eqref{eq:help-13} can be satisfied by selecting constant $c$ such that:
\begin{equation}
    0 < c \leq \frac{\sqrt{B_t^2 + 4A_tC_t}- B_t }{2A_t},
\end{equation}
which concludes the proof.
\qed

\section{Proof of Corollary \ref{corollary1}}\label{proof:Corollary1}

In order to show the convergence, the key is to prove with the selected learning rates, the RHS of \eqref{eq:thm-mimic} diminishes to zero with a high probability.
It is straightforward that $\lim_{T\rightarrow \infty} \frac{ 2 }{\sum_{t=0}^{T-1} \eta_t} (\mathbb{E} [f(\mathbf{w}_{0})] - \mathbb{E} [f(\mathbf{w}^*)]) = 0$ as the learning rates satisfy $\lim_{T\rightarrow \infty} \sum_{t=0}^{T-1} \eta_t = \infty$.
Besides, as $\eta_L = \mathcal{O} \left(\frac{1}{L\sqrt{T}} \right)$, we have $\lim_{T \rightarrow \infty} \varphi_K = (\frac{2^K-1}{K}+1)L^2$, which is irrelevant of $t$.
Hence, we complete the proof as follows:
\begin{align}
    & \lim_{T \rightarrow \infty} \frac{2}{\sum_{t=0}^{T-1} \eta_t}  \sum_{t=0}^{T-1} \frac{\sigma^2}{|\mathcal{S}_t|K} \nonumber \\ 
    & \quad \times \left( \frac{\eta_t^2 L}{2} + \frac{\rho_t(\rho_t-\nu) \eta_t \varphi_K \tau_{\max}^2}{\nu(\rho_t-1-\nu)} \frac{\eta_{t+1}^2}{|\mathcal{S}_{t+1}|} \right) \nonumber \\
    = &\lim_{T\rightarrow \infty} \mathcal{O} \left( \frac{1}{\sum_{t=0}^{T-1} \eta_t} \sum_{t=0}^{T-1} \eta_t^2 \right) \nonumber \\
    = & 0.
\end{align}
\qed

\section{Proof of Theorem \ref{high-prob}} \label{proof:high-prob}
Denote $\Tilde{\tau}(t) \triangleq \max_{i\in \mathcal{S}_t} \tau(t,i)$ as the maximum dropout iterations of active clients in iteration $t$.
Following a similar proof of Theorem \ref{thm:mimic}, we can arrive at inequality \eqref{eq:help-2}.
Dividing both sides of Step (b) in \eqref{eq:help-2} by $\sum_{t=0}^{T-1} \frac{\eta_t}{2}$ yields the following inequality: 
\begin{align}
    & \frac{1}{\sum_{t=0}^{T-1} \eta_t} \sum_{t=0}^{T-1} \eta_t \mathbb{E} [\left\| \nabla f(\mathbf{w}_t) \right\|^2] \nonumber \\
    \leq & \frac{ 2 }{\sum_{t=0}^{T-1} \eta_t} (\mathbb{E} [f(\mathbf{w}_{0})] - \mathbb{E} [f(\mathbf{w}^*)]) \nonumber \\
    + & \frac{2}{\sum_{t=0}^{T-1} \eta_t} \sum_{t=0}^{T-1} \frac{\sigma^2}{|\mathcal{S}_t|K} \left( \frac{\eta_t^2 L}{2} \!+\! \frac{\rho_t(\rho_t-\nu) \eta_t \varphi_K}{\nu(\rho_t-1-\nu)} \Tilde{\tau}^2(t) \frac{\eta_{t+1}^2}{|\mathcal{S}_{t+1}|}\right). \label{eq:help-3}
\end{align}

We then show the relationship between the active probabilities $\{p_{i}^{t}\}$'s and $\Tilde{\tau}(t)$ via the following result adapted from the proof of Theorem 5.2 in \cite{mifa}.
Given any $\delta >0$, with probability at least $1 - \delta$, the following inequality holds:
\begin{equation}
    \Tilde{\tau}(t) \leq \frac{\lambda}{\min_{i}p_t^i} \left(\log \left( \frac{Nt}{\delta} \right)+1 \right),
    \label{eq:connect}
\end{equation}
with a constant $\lambda>0$, which indicates that $\Tilde{\tau}(t) = \mathcal{O}\left(\log \left( \frac{t}{\delta} \right)\right)$.
Therefore, by substituting the RHS of \eqref{eq:connect} for $\Tilde{\tau}(t)$ in \eqref{eq:help-3}, we complete the proof. 
\qed

\section{Proof of Corollary \ref{Corollary:high-prob}}\label{proof:Corollary2}

In order to show the convergence, the key is to prove with the selected learning rates, the RHS of \eqref{eq:thm-mimic-2} diminishes to zero with a high probability.
It is straightforward that $\lim_{T\rightarrow \infty} \frac{ 2 }{\sum_{t=0}^{T-1} \eta_t} (\mathbb{E} [f(\mathbf{w}_{0})] - \mathbb{E} [f(\mathbf{w}^*)]) = 0$ as the learning rates satisfy $\lim_{T\rightarrow \infty} \sum_{t=0}^{T-1} \eta_t = \infty$.
In addition, for a given $\delta >0$, the following equality
\begin{align}
    &\lim_{T\rightarrow \infty} \frac{1}{\sum_{t=0}^{T-1} \eta_t} \sum_{t=0}^{T-1} \frac{2\sigma^2}{|\mathcal{S}_t|K} \nonumber \\
    & \times \left( \eta_t^2 L^2 + \frac{\rho_t(\rho_t-\nu) \eta_t \varphi_K}{\nu(\rho_t-1-\nu)} \Tilde{\tau}^2(t) \frac{\eta_{t+1}^2}{|\mathcal{S}_{t+1}|} \right) \nonumber \\
    \overset{(\text{a})}{=} &\lim_{T\rightarrow \infty} \mathcal{O} \left( \frac{1}{\sum_{t=0}^{T-1} \eta_t} \sum_{t=0}^{T-1} \eta_t^2 \right) \nonumber \\
    &+ \lim_{T\rightarrow \infty} \mathcal{O} \left( \frac{1}{\sum_{t=0}^{T-1} \eta_t} \sum_{t=0}^{T-1} \eta_{t}^3 \left(\log \left( \frac{t}{\delta} \right)\right)^2 \right), \label{eq:help-4}
\end{align}
holds with probability at least $1 - \delta$, where (a) follows $\Tilde{\tau}(t) = \mathcal{O}\left(\log \left( \frac{t}{\delta} \right)\right)$.
It is also easy to show $\lim_{T\rightarrow \infty} \mathcal{O} \left( \frac{1}{\sum_{t=0}^{T-1} \eta_t} \sum_{t=0}^{T-1} \eta_t^2 \right) = 0$.
Note that $\delta$ is irrelevant of $t$.
Besides, if $\eta_t = \mathcal{O} \left(\frac{|\mathcal{S}_t|}{t+\beta} \right)$ with $\beta>0$, the second term in the RHS of \eqref{eq:help-4} becomes $\lim_{T\rightarrow \infty} \mathcal{O} \left( \frac{1}{\sum_{t=0}^{T-1} \frac{1}{t+\beta}} \sum_{t=0}^{T-1} \frac{t}{(t+\beta)^3} \right) = 0$.
In other words, MimiC converges to a stationary point of the global loss function with probability at least $1-\delta$.
\qed

\end{document}